\documentclass[11pt, letterpaper, shortlabels]{berkeley}

\usepackage{amsmath,amsfonts,bm}

\def\eqref#1{equation~\ref{#1}}

\def\1{\bm{1}}

\DeclareMathAlphabet{\mathsfit}{\encodingdefault}{\sfdefault}{m}{sl}
\SetMathAlphabet{\mathsfit}{bold}{\encodingdefault}{\sfdefault}{bx}{n}

\usepackage{amsmath, amsthm, amssymb, mathtools, mathrsfs}
\usepackage{graphicx, booktabs, multirow, wrapfig, subcaption}
\usepackage{enumitem, array, url, microtype, dsfont, nicefrac}
\usepackage{tcolorbox}
\tcbuselibrary{skins, breakable, listings, theorems}
\usepackage{setspace, lipsum, fontawesome5}
\usepackage{natbib}
\usepackage{soul}

\usepackage{algorithm}
\usepackage{algpseudocode}  

\definecolor{NDblue}{RGB}{12, 35, 64}
\definecolor{NDgold}{RGB}{174, 145, 66}
\definecolor{darkblue}{rgb}{0, 0, 0.5}
\definecolor{deepblue}{rgb}{0,0,0.5}
\definecolor{deepred}{rgb}{0.6,0,0}
\definecolor{deepgreen}{rgb}{0,0.5,0}
\definecolor{NDblue}{RGB}{12, 35, 64}
\definecolor{NDgold}{RGB}{174, 145, 66}
\definecolor{bestcell}{RGB}{220,255,220}
\definecolor{modelrow}{gray}{0.95}
\definecolor{promptgray}{RGB}{200,200,200}
\definecolor{promptblue}{RGB}{25,118,210}
\PassOptionsToPackage{table}{xcolor}  
\usepackage[table]{xcolor}  
\usepackage{hyperref}

\hypersetup{
    colorlinks = true,
    linkcolor = NDblue,
    citecolor = NDgold,
    urlcolor  = NDblue,
    filecolor = NDblue
}

\usepackage[capitalize,noabbrev]{cleveref}
\Crefformat{equation}{#2Eq.~(#1)#3}
\Crefformat{figure}{#2Figure~#1#3}
\Crefformat{assumption}{#2Assumption~#1#3}
\Crefname{assumption}{Assumption}{Assumptions}
\usepackage{crossreftools}
\newtheorem{theorem}{Theorem}

\usepackage{listings}
\lstdefinestyle{pythonstyle}{
    basicstyle=\ttfamily\footnotesize,
    language=Python,
    keywordstyle=\color{deepblue},
    stringstyle=\color{deepgreen},
    frame=single,
    showstringspaces=false,
}
\lstnewenvironment{python}[1][]{
    \lstset{style=pythonstyle,#1}
}{}

\newtcolorbox{promptbox}[2][]{%
    enhanced,
    unbreakable,
    before skip=2mm,
    after skip=2mm,
    colback=darkblue!5!white,
    colframe=darkblue,
    coltitle=white,
    boxrule=0.5mm,
    sharp corners,
    arc=5pt,
    attach boxed title to top center={yshift=-3mm},
    boxed title style={
        enhanced,
        colback=NDgold,
        colframe=darkblue,
        arc=5pt,
        outer arc=5pt,
        boxrule=0pt,
    },
    title={\faLightbulb[solid]\space #2},
    fonttitle=\bfseries\color{white},
    #1
}

\newcommand{\SB}{$\mathrm{BERTScore}$}
\newcommand{\SZX}{$\mathrm{S(Z,X)}$}
\newcommand{\SZY}{$\mathrm{S(Z,Y)}$}
\newcommand{\SXY}{$\mathrm{S(X,Y)}$}

\title{Causally-Enhanced Reinforcement Policy Optimization }
\author{Xiangqi Wang}
\author{Yue Huang}
\author{Yujun Zhou}
\author{Xiaonan Luo}
\author{Kehan Guo}
\author{Xiangliang Zhang}

\affil{University of Notre Dame \par
  \href{mailto:xwang5@nd.edu}{\texttt{xwang76}},\;
  \href{mailto:yhuang6@nd.edu}{\texttt{yhuang37}},\;
  \href{mailto:yzhou4@nd.edu}{\texttt{yzhou25}},\;
  \href{mailto:xluo2@nd.edu}{\texttt{xluo6}},\;
  \href{mailto:kguo3@nd.edu}{\texttt{kguo2}},\;
  \href{mailto:xzhang33@nd.edu}{\texttt{xzhang33}}
  \texttt{@nd.edu}}
\correspondingauthor{xzhang33@nd.edu (Xiangliang Zhang)}

\begin{abstract}
Large language models (LLMs) trained with reinforcement objectives often achieve superficially correct answers via shortcut strategies, pairing correct outputs with spurious or unfaithful reasoning and degrading under small causal perturbations. We introduce \emph{Causally-Enhanced Policy Optimization} (CE-PO), a drop-in reward-shaping framework that augments policy optimization with a differentiable proxy for causal coherence along the generation pathway from prompt ($Z$) to rationale ($X$) to answer ($Y$). CE-PO estimates model-internal influence with Jacobian-based sensitivities, counterfactually hardens these signals to suppress nuisance cues, and fuses the resulting coherence score with task-accuracy feedback via a Minkowski (power-mean) combiner, exposing a single tunable between accuracy and coherence trade-off. The unified reward integrates with PPO/GRPO without architectural changes. Across reasoning benchmarks and causal stress tests, CE-PO reduces reward hacking and unfaithful chain-of-thought while improving robustness to correlation--causation flips and light counterfactual edits, all at near-parity accuracy. Experimental results across 4 datasets show that CE-PO improves accuracy over baselines by 5.49 $\%$ on average (up to 9.58 $\%$), while improving robustness to correlation–causation flips and light counterfactual edits.

\end{abstract}
\begin{document}
\maketitle

\section{Introduction}\label{sec:intro}


Outcome-centric reinforcement training of large language models (LLMs)—whether via standard policy optimization or preference-based tuning~\citep{Christiano2017deep,Ouyang2022training,Schulman2017PPO, Shao2024DeepSeekMath} optimizes \textit{what} answer is produced but places little pressure on \textit{how} that answer is obtained~\citep{sim2025lessons}. As a result, models routinely learn to “game” reward signals (e.g., by exploiting length, format, or lexical overlap) and to pair correct answers with unfaithful or spurious reasoning traces~\citep{wen2024language, Arjovsky2019invariant}. This failure mode is not merely theoretical: without intervention, we observe models can game both accuracy proxies and \emph{gradient-based sensitivity} surrogates—e.g., by guessing final answer or maximizing output length as in~\autoref{sec:experiments}. This mismatch between scoring for outcomes and expecting coherent reasoning harms robustness, causing severe and threatening hallucination.


We argue that robust reasoning requires not only accurate outputs but also \emph{causal coherence} along the generation pathway. By \emph{coherence} we mean that (i) the prompt genuinely informs the rationale, (ii) the prompt along with rationale in turn genuinely informs the final answer, and (iii) small, causally relevant edits to the prompt or rationale induce correspondingly appropriate changes downstream, whereas irrelevant edits do not. Concretely, let $(Z\!\rightarrow\!X\!\rightarrow\!Y)$ denote the intended chain from prompt ($Z$), to rationale ($X$), to answer ($Y$)~\citep{Bao2024NonCausalCoT}. If the model’s internal influence respects this order, chain-of-thought (CoT)~\citep{Wei2022chain} traces are more likely to be right for the right reasons~\citep{Ouyang2022training, Shao2024DeepSeekMath} and to generalize under distribution shift. However, current reinforcement objectives offer no internal signal that rewards such coherence, as previous literature on causal reward in LLM typically focus on causally indifferent to external factors instead of reasoning coherence~\citep{wang2025beyond}. Related efforts to promote faithful reasoning typically rely on external supervision or decoding-time control: rationale supervision and “right-for-the-right-reasons” regularizers require human explanations and usually target only input and generation; concept-bottleneck~\citep{KohLiang2017} and invariance/IRM approaches~\citep{pmlr-v130-kamath21a} enforce intermediate structure but do not address generative $prompt\!\to\!rationale\!\to\!answer$ flows; CoT verification and constrained decoding act \textit{post hoc}~\citep{Chi2025G2Reasoner, Heimersheim2024ActivationPatching} rather than shaping training. In contrast, we seek a \emph{differentiable, annotation-free, training-time} signal that explicitly rewards influence coherence across $Z\!\to\!X\!\to\!Y$ and integrates into standard policy optimization. 

Translating the above goal into a trainable signal raises three coupled issues. \emph{Measurability}—existing faithfulness tools (saliency, raw input-gradient attributions) are fragile and can fail basic sanity checks or depend weakly on the learned model, limiting their usefulness as training signals \citep{adebayo2018sanity}; gradient regularization that makes models “right for the right reasons” typically relies on human-provided rationales and targets only input $\!\to\!$ output links in discriminative settings, not the generative prompt $\!\to\! $ rationale $\!\to\!$ answer pathway \citep{ross2017right}, while causal-reasoning benchmarks show LLMs still conflate correlation and causation \citep{Jin2024corr2cause}. This motivates a \emph{differentiable, model-internal} proxy for directed influence along $Z\!\to\!X$, $X\!\to\!Y$, and $Z\!\to\!Y$ that is compatible with policy-gradient optimization. \emph{Robustness to spurious sensitivity}—naïve gradient signals are easily inflated by non-semantic factors (token frequency, position, format), encouraging shortcut learning \citep{geirhos2020shortcut} and \emph{reward hacking} under outcome-only RL objectives \citep{amodei2016concrete}; moreover, CoT can appear compelling yet remain unfaithful and brittle under small perturbations \citep{Wei2022chain,Bao2024noncausal,Chi2024Mirage}. Hence we require \emph{counterfactual hardening} that explicitly breaks semantic links and filters out nuisance directions rather than trusting raw sensitivities. \emph{Objective balance}—robustness-oriented methods (e.g., invariance penalties) often trade off in-distribution accuracy and can underperform ERM when mis-specified \citep{Arjovsky2019invariant,rosenfeld2020risks,pmlr-v130-kamath21a}, whereas optimizing only task reward preserves shortcut exploitation \citep{amodei2016concrete}. We therefore need a \emph{tunable} mechanism to balance accuracy against coherence during optimization, avoiding both over-regularization and reward hacking.

\noindent\textbf{Our approach.}
We introduce \underline{C}ausally-\underline{E}nhanced Policy Optimzation(CE-PO)~\footnote{Code implementation is available at: https://github.com/XiangqiWang77/causalrl}, a drop-in reward-shaping scheme that adds a \emph{Jacobian-based causal-coherence signal} to standard policy optimization and fuses it with task accuracy via a \emph{Minkowski (power-mean) combiner}. Concretely, we compute model-internal Jacobian influence signals for the $prompt \!\to\! rationale$ and $rationale\!\to\!answer$ links; harden them with a simple counterfactual procedure (removing the resulting nuisance directions by reshuffling) to reduce sensitivity to superficial cues; normalize and aggregate into a single \emph{coherence score}; and combine this score with accuracy using a power-mean with tunable weights/exponent, yielding a single reward that trades off coherence and accuracy. The unified reward is plugged into PPO/GRPO without architectural changes. In summary, our contributions are threefold:
\begin{itemize}[noitemsep, leftmargin=*]
  \item We propose a differentiable, model-internal \emph{DCE-proxy} (influence–coherence) reward along $Z\!\to\!X\!\to\!Y$, coupled with \emph{counterfactual residualization} (see \autoref{subsec: jacob1}) to suppress shortcut sensitivities and \emph{reward hacking}.
  \item We introduce a \emph{Minkowski (power-mean) combiner} (detailed in \autoref{subsec:RLTradeoff}) that exposes a tunable trade-off between task accuracy and causal-coherence signals, serving as a drop-in objective for PPO/GRPO with efficient reward computation.
  \item  CE-PO demonstrates consistent performance improvements, supported by comprehensive ablation analyses and competitive out-of-distribution generalization. In addition, our study highlights notable generation patterns and offers insights into model behavior.
\end{itemize}

\vspace{-5pt}
\section{Causal Enhanced Policy Optimization}
\label{sec:method}

\paragraph{Overview.}
In this section, we give implementation details of \emph{CE-PO} and formalize it. Given a rollout with prompt $Z$, rationale $X$, and answer $Y$, as shown in Figure~\ref{fig:wholepipeline}. The method unfolds in three steps that map one-to-one to the subsections below.
(i) \textbf{Raw influence signals} : we compute \emph{Base Jacobians}, i.e., local sensitivities of downstream span likelihoods with respect to upstream token embeddings, as differentiable, model-internal proxies for $Z\!\to\!X$, $X\!\to\!Y$, and for $Z\!\to\!Y$ removing mediator $X$.
(ii) \textbf{Counterfactual hardening} : Deploying raw Jacobian scores as rewards risks \emph{reward hacking}, where spurious factors (e.g., length) dominate the signal~\autoref{fig:genlen-training}. To suppress such shortcut and formatting sensitivities, we break semantic links (reshuffle/mismatch) and remove the induced nuisance directions from the raw signals.
(iii) \textbf{Reward normalization and fusion} We normalize Jacobian responses into a stable scalar \emph{coherence score} and then fuse it with task accuracy using a Minkowski (power-mean) combiner, yielding a unified reward that exposes a tunable accuracy–coherence trade-off and is optimized with PPO/GRPO.
This section defines each component, provides complexity notes, and clarifies how they compose into the training loop.

\begin{figure}[t]
    \centering
    \vspace{-10pt}
    \includegraphics[width=\textwidth]{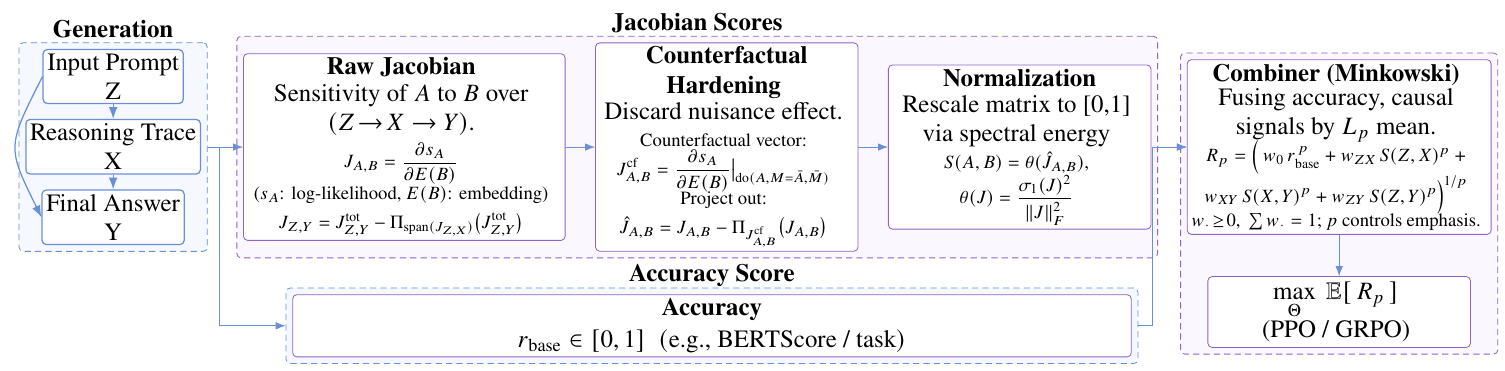}
    \vspace{5pt}
    \caption{\textbf{CE-PO pipeline (left $\rightarrow$ right).}
\emph{Generation:} $Z\!\to\!X\!\to\!Y$ (X/Y split by a special token; see Fig.~\ref{fig:finalanswer_instructions}).
\emph{Base Jacobians:} local sensitivities for $Z{\to}X$, $X{\to}Y$, $Z{\to}Y$.
\emph{Counterfactual hardening:} build source-side counterfactuals and project their subspace from the base signals to obtain hardened residuals. 
\emph{Normalization: } use spectral energy to convert Jacobian matrix into [0,1] scalar.
\emph{Fusion:} combine $S(A,B)$ signals with accuracy $r_{\text{base}}$ via a Minkowski combiner; optimize the unified reward with PPO/GRPO.}
    \vspace{-15pt}
    \label{fig:wholepipeline}
\end{figure}

\subsection{Jacobian Signals, and Causally-Enhanced Reward}
\label{subsec: jacob1}
We consider a differentiable language model $\Theta$ and our goal is, as previously described, to improve the accuracy on $Y$ (i.e., $r_{base})$ yet simultaneously to ensure the alignment of influence with this order: $Z$ should shape $X$, and $X, Z$ should shape $Y$. 
Our proposed method is illustrated in \autoref{fig:wholepipeline}, providing a causally grounded reward signal $R_p$ by combining $r_{base}$ and Jacobian-based scores measuring causal effects along $Z\!\rightarrow\!X\!\rightarrow\!Y$.  


\vspace{-10pt}
\paragraph{Base Jacobians.}
While $Y$-accuracy rewards \emph{what} answer is produced, it does not constrain \emph{how} the answer is obtained. To shape the reasoning pathway, we introduce a differentiable, model-internal \emph{Direct Causal Effect (DCE) proxy}~\citep{Pearl2009,VanderWeele2011} for coherence along $Z\!\to\!X\!\to\!Y$: small and causal changes in upstream tokens should induce appropriate changes downstream. Concretely, let $s_X$ and $s_Y$ denote the log-likelihoods of the rationale span $X$ and the answer span $Y$ posterior to $Z$ and $Z,X$ as in \autoref{eq:equation2}, and let $E(A)\!\in\!\mathbb{R}^{T_A\times d}$ be the token-embedding matrix for a text segment $A$ (length $T_A$, hidden size $d$) and $a_{<t}$ denotes the prefix strictly before index $t$.
\begin{equation}
s_X \;=\; \sum_{t\in X}\log p_\Theta\!\big(x_t \mid Z, x_{<t}\big), \qquad
s_Y \;=\; \sum_{t\in Y}\log p_\Theta\!\big(y_t \mid Z, X, y_{<t}\big)
\label{eq:equation2}
\end{equation}
We measure the influence from $A$ to the score of $B$ by the \emph{block Jacobian} $J_{AB}\!\triangleq\!\partial s_B/\partial E(A)$. For the intended chain $Z\!\rightarrow\!X\!\rightarrow\!Y$, the \emph{Base Jacobians} (raw, unhardened sensitivities) are
\begin{equation}
J_{ZX}=\frac{\partial s_X}{\partial E(Z)}\in\mathbb{R}^{T_Z\times d},\qquad
J_{XY}=\frac{\partial s_Y}{\partial E(X)}\in\mathbb{R}^{T_X\times d},\qquad
J_{ZY}^{\mathrm{total}}=\frac{\partial s_Y}{\partial E(Z)}\in\mathbb{R}^{T_Z\times d}.
\label{eq:base-jacobians}
\end{equation}
In addition, to approximate the Direct Causal Effect (DCE) from $Z$ to $Y$, we need to isolate X effect out of $J_{ZY}^{\mathrm{total}}$. We remove the mediated path via projection: 
forming a via-\(X\) template from \(J_{ZX}\) (\(Z\!\to\!X\)) and the principal direction of \(J_{XY}\) (\(X\!\to\!Y\)), then subtracting its projection from the total effect  to obtain the direct effect:
\begin{equation}J_{ZY}=J_{ZY}^{\mathrm{total}}-\Pi_{\,J_{ZX}\odot\bar u}\!\left(J_{ZY}^{\mathrm{total}}\right),
\label{eq:isolate}\end{equation}
where $\bar u=\tfrac{1}{T_X}\sum_{t\in X}\tfrac{J_{XY}[t,:]}{\|J_{XY}[t,:]\|_2}$, 
and $\Pi_A(B)=\tfrac{\langle B,A\rangle_F}{\langle A,A\rangle_F}\,A$,  with $\langle B,A\rangle_F$ indicating the Frobenius inner product  between $A$ and $B$, and $\odot$ be Hadamard product.

We use the Jacobian as a local, direct causal effect \emph{proxy}: for any upstream–downstream pair $(A,B)\in\{(Z,X),(X,Y),(Z,Y)\}$, with mediators held at their observed values, $J_{AB}=\partial s_B/\partial E(A)$ quantifies how an infinitesimal perturbation to $A$’s embeddings shifts the downstream log-likelihood $s_B$—i.e., which directions in $A$ immediately move $B$, and by how much. Crucially, $J_{AB}$ is not equal to a causal DCE in the mediation sense; it is the \emph{first-order Taylor approximation} around the observed point~\citep{Pearl2001}. They are already informative enough and avoid the cost and instability of higher-order (Hessian/influence-function) estimators in deep/LLM settings~\citep{Pearlmutter1994,Martens2010,KohLiang2017,li2024influence}.

We use $J_{AB}$ to serve as the starting point for our \emph{coherence} signal. In the next subsections we apply \emph{counterfactual hardening} to remove nuisance directions and then normalize/aggregate the residual influence into a scalar coherence score that is fused with task accuracy.





\vspace{-10pt}
\paragraph{Counterfactual-hardened Jacobian score.}
Jacobian-based rewards, when computed directly on the base input, may conflate causal content with nuisance factors such as style, formatting and token-frequency statistics. As a result, innocuous paraphrases or vacuous elongation can inflate the gradient norm and concentrate its top singular direction, enabling \emph{length-driven reward hacking} up to the token limit~\citep{gao2023scaling,rafailov2024scaling}. Empirically, as shown in \autoref{fig:genlen-training} (blue), the length hacking existing in non-counterfactual Jacobian rewards (Non-CF-GRPO conflates).

To suppress shortcut-driven sensitivities in the \emph{Base Jacobians}, we harden them with pair-specific counterfactuals.
For each link $(A,B)\in\mathcal{L}:=\{(Z,X),(X,Y),(Z,Y)\}$, let $M(A,B)$ denote mediators on $A{\to}B$ paths (empty except $M(Z,Y)=\{X\}$).
We break the semantic alignment between $A$ (and $M$) and $B$ by permuting the tokens of $A$ (and of $M$ when present), yielding $(\bar A,\bar M)$ that preserves surface statistics (length/frequency) while being approximately independent of $B$.
Let $s_B$ be the log-likelihood of span $B$ and $E(A)\!\in\!\mathbb{R}^{T_A\times d}$ the embedding matrix of $A$.
We compute the \emph{counterfactual Jacobian}
\begin{equation}
J_{AB}^{\mathrm{cf}}=\left.\frac{\partial s_B}{\partial E(A)}\right|_{(A,M)\mapsto(\bar A,\bar M)}.
\label{eq:contfac1}
\end{equation}
We   construct \(K\) \emph{source-side} counterfactuals  and calculate their Jacobians  by \autoref{eq:contfac1}, which are then averaged to have \(\bar J^{\text{cf}}_{AB}\)  (we set \(K\)=4, larger values generally improve results but incur higher computational cost).

Then, we extract the principal counterfactual subspace via the top-$k$ left singular vectors $U_{AB}$ of $J_{AB}^{\mathrm{cf}}$, define the projector $\Pi_{AB}(J)=U_{AB}U_{AB}^{\top}J$, following \citep{ravfogel2020null}. Taking the $J_{AB}$  in \autoref{eq:base-jacobians} and \autoref{eq:isolate}, we obtain the deconfounded (direct) block by removing these directions:
\begin{equation}
\widehat{J}_{AB}=J_{AB}-\Pi_{AB}\!\big(J_{AB}\big).
\label{eq:projector}
\end{equation}
Permuted signals expose a spurious nuisance subspace; projecting it out yields a deconfounded, causally aligned signal that resists surface correlations and reward hacking \citep{ravfogel2020null}.

As shown in Figure ~\ref{fig:genlen-training} (purple), the Non\mbox{-}CF\mbox{-}GRPO baseline drives the mean generation length to the token cap, revealing a length\mbox{-}based reward\mbox{-}hacking loop. With our counterfactual reshuffling and projection (CE\mbox{-}GRPO), the curve stays well below the cap and does not plateau, while reward improves. The mechanism is simple: source\mbox{-}side counterfactuals preserve surface statistics (length/frequency) while scrambling semantics \citep{ravfogel2020null}, so the counterfactual Jacobian spans a nuisance subspace; orthogonally projecting the base Jacobian off this subspace—akin to Neyman orthogonalization \citep{oprescu2019orthogonal}—decouples the reward from length/format sensitivities and correlations \citep{chernozhukov2018double}. For the $Z{\to}Y$ link, we additionally subtract the via\mbox{-}$X$ component before projection (a first\mbox{-}order direct\mbox{-}effect correction), further reducing spurious incentives.


\begin{figure}[t]
  \centering
  \vspace{-5pt}
  \includegraphics[width=0.8\linewidth]{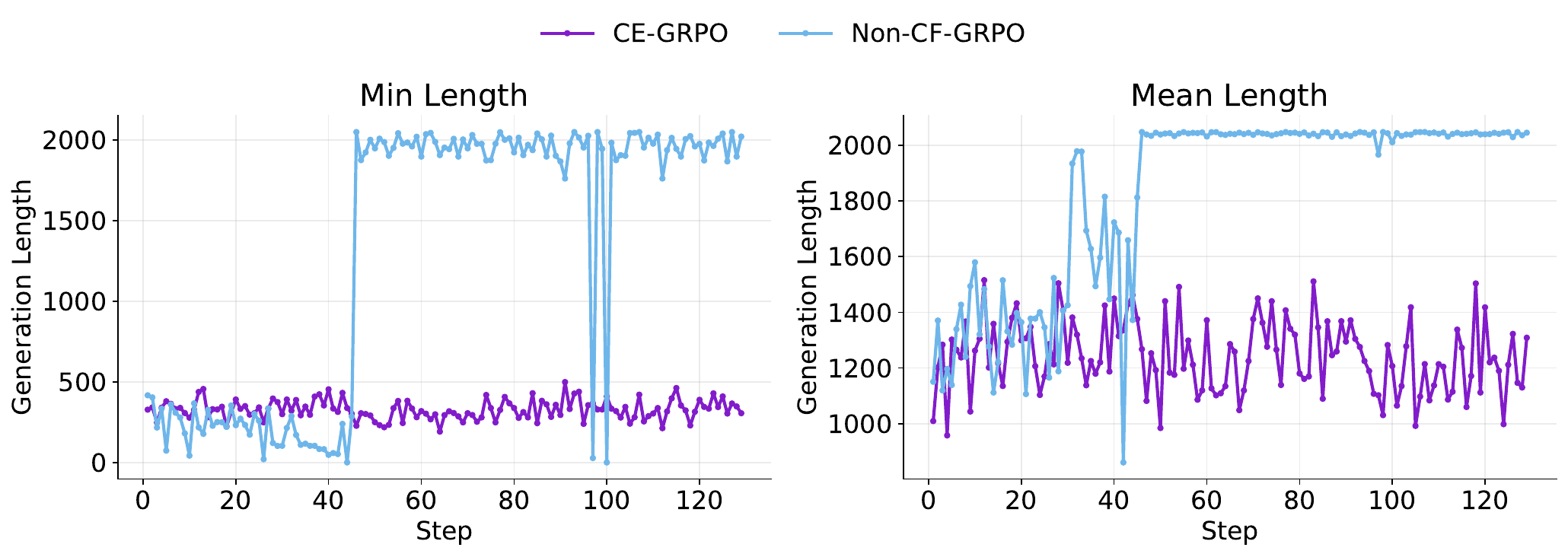}
  \vspace{0pt}
  \caption{Training trajectories comparison:    \textbf{Non-CF-GRPO} (GRPO with non-CF award) vs \textbf{CE-GRPO} (GRPO with our CE award)  on Qwen-3-4B-Thinking with \texttt{max\_tokens}=2048, in terms of minimum (left) and mean (right) generation length. Non-CF-GRPO curves plateau as generations approach the token limit, indicating length-driven reward hacking, while  CE-GRPO doesn't.}
  \vspace{-10pt}
  \label{fig:genlen-training}
\end{figure}

\vspace{-10pt}
\paragraph{Reward normalization and fusion}

Finally, we \emph{scalarize} the block to a score via the scale-free spectral energy share $S(A,B)=\phi(\widehat{J}_{AB})$, where $\phi(J)=\sigma_1(J)^2/\lVert J\rVert_F^2\in[0,1]$ to squeeze and standardize the process \citep{vershynin2018high}.
Concretely,
\begin{equation}
S(Z,X)=\phi\!\big(\widehat{J}_{ZX}\big),\qquad
S(X,Y)=\phi\!\big(\widehat{J}_{XY}\big),\qquad
S(Z,Y)=\phi\!\big(\widehat{J}_{ZY}\big).
\end{equation}

Let $r_{\text{base}}\!\in\![0,1]$ be a task accuracy score on $Y$ (e.g., semantic correctness, measured by   BERTScore \citep{zhang2019BERTScore} in our setting). We define the new causally grounded reward by  integrating accuracy and the three counterfactual-enhanced signals via a weighted power mean (Minkowski combiner)~\citep{bullen2013handbook},
\begin{equation}
R_p=\Big(w_0\,r_{\text{base}}^{\,p}+w_1\,S(Z,X)^{p}+w_2\,S(X,Y)^{p}+w_3\,S(Z,Y)^{p}\Big)^{1/p},
\end{equation}
with $p\in\mathbb{R}$, weights $w_i\!\ge\!0$, and $\sum_{i=0}^3 w_i=1$. As $p\!\to\!0$, $R_p$ approaches a weighted geometric mean (minimal optimization); larger $p$ emphasizes the largest channel, which respectively matches deconfounded Jacobian causal scores and BERTScore as in \autoref{subsec:RLTradeoff}. Using this $R_p$ in any common policy-gradient training  of LLM (e.g., PPO, GRPO), the goal is to maximize $\mathbb{E}[R_p]$.

\textit{Note of $r_{\text{base}}$:}
We set $r_{\text{base}}$ by 
BERTScore (i.e., the cosine similarity between $Y$ and ground truth in BERT embedding space), as it is a continuous measure between 0 and 1 for accuracy, aligning with the scale of the causal term. 
In contrast, 0/1 judgments are often sensitive to judging prompts and  are less compatible; mixing such binary signals with the causal term can induce oscillatory updates.
Ablation studies on replacing BERTScore with 0/1 rewards  
are provided in \autoref{subsec:ablation}. In addition, we also provide TRPO-form ~\citep{schulman2015trust} lower bound for CE-PO with details  in \autoref{sec:theory} and the discussion for training efficiency can be found in \autoref{app:additionaltechnique}. 

\paragraph{Training objective}
Let $R_p(Z,X,Y)\!\in\![0,1]$ denote the unified reward from Eq.~(7).
We treat $R_p$ as a \emph{scalar terminal reward} for each rollout $(Z,X,Y)$:
$r_T\!=\!R_p,\ r_t\!=\!0$ for $t{<}T$, and \textbf{do not backpropagate through $R_p$}
(standard policy-gradient; the Jacobian signals are used only to compute $R_p$).

\textbf{CE-PPO.}
With behavior policy $\pi_{\text{old}}$, importance ratio $w_t(\theta)=\pi_\theta(a_t|s_t)/\pi_{\text{old}}(a_t|s_t)$,
and GAE advantages $\hat A_t$ computed from the sparse reward, the clipped objective is
\[
\mathcal{J}_{\text{CE-PPO}}(\theta)
=\mathbb{E}\Big[\min\big(w_t(\theta)\hat A_t,\ \mathrm{clip}(w_t(\theta),1{-}\epsilon,1{+}\epsilon)\hat A_t\big)
-\beta\,\mathrm{KL}\big(\pi_\theta\,\|\,\pi_{\text{ref}}\big)\Big].
\tag{8}
\]
We optionally standardize $R_p$ per batch to stabilize advantages.

\textbf{CE-GRPO.}
For each prompt we sample $K$ candidates, compute $R_{p,k}$, and form
group-relative advantages $\tilde A_k=\big(R_{p,k}-\mu_{\text{grp}}\big)/(\sigma_{\text{grp}}+\varepsilon)$,
with $\mu_{\text{grp}},\sigma_{\text{grp}}$ the group mean/std.
The GRPO-style objective is
\[
\mathcal{J}_{\text{CE-GRPO}}(\theta)
=\mathbb{E}\Big[\tfrac{1}{K}\sum_{k=1}^K \min\big(w_k(\theta)\tilde A_k,\ \mathrm{clip}(w_k(\theta),1{-}\epsilon,1{+}\epsilon)\tilde A_k\big)
-\beta\,\mathrm{KL}\big(\pi_\theta\,\|\,\pi_{\text{ref}}\big)\Big].
\tag{9}
\]

\subsection{Reward Signals Trade-off}

Our RL objective (maximizing   $\mathbb{E}[R_p]$) blends two signals,  semantic accuracy (e.g., BERTScore) and causal coherence (counterfactual-enhanced Jacobian). These signals misalign at times, mirroring RLHF’s trade-off between surface correctness and  shortcut avoidance \citep{kalai2025languagemodelshallucinate,lightman2023let}. By defining $R_p$ via the Minkowski norm, we enable a dynamically tunable balance between the two signals. 

\label{subsec:RLTradeoff}
\begin{wrapfigure}[7]{r}{0.3\textwidth}
  \centering
  \vspace{-22pt}
  \includegraphics[width=0.3\textwidth]{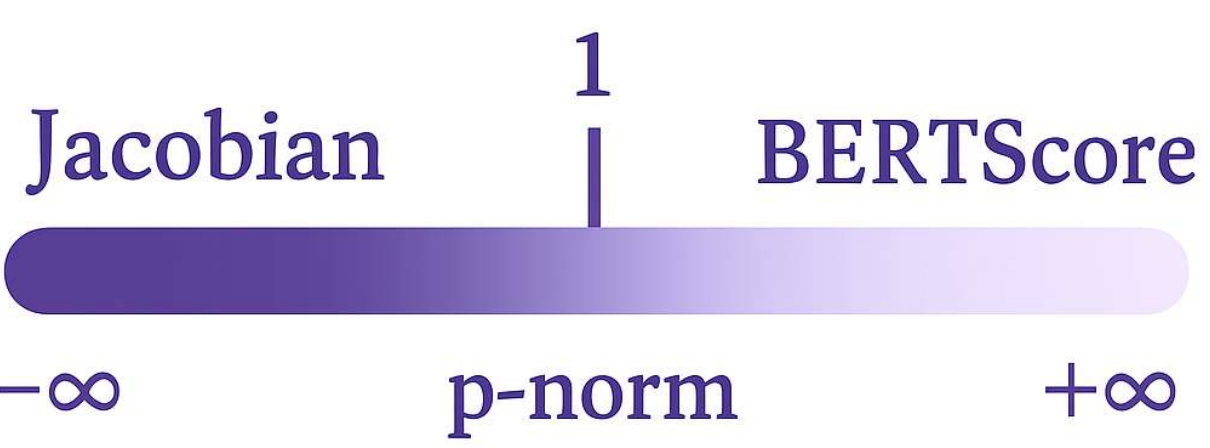}
  \vspace{-12pt}
  \caption{Tunable balance enabled by Minkowski
 $p$-norm interpolation in $R_p$.  
  }
  \label{fig:pnorm}
\end{wrapfigure}

\autoref{fig:pnorm} illustrates how the  Minkowski $p$-norm interpolation provides a smooth mechanism to mediate the trade-off between semantic accuracy and causal coherence. As $p \to +\infty$,  
the reward collapses to the larger signal,  typically
BERTScore, which is unprojected and numerically higher. This  
drives accuracy-focused optimization (as observed in all
models). Conversely, as $p \to -\infty$, the reward collapses to the smaller signal (e.g., when $p$ =$0$ the interpolation  equals geometric mean), which is usually the counterfactual Jacobian, lowered by residualization. This, in turn, emphasizes causal coherence. This multi-objective
behavior parallels prior work on balancing conflicting RL rewards \citep{won2019projection, roijers2013survey}. To illustrate how CE-PO combine both signals can thus enhance LLM reasoning capability to be both causally capable yet not template-following, we provide an LLM generation example in \autoref{tab:colorbox}.

\section{Experiments}
\label{sec:experiments}
In this section, we show performance of CE-PO on top of PPO and GRPO (noted as \textbf{CE-PPO} or \textbf{CE-GRPO}) against simple baselines, using BERTScore as the reward (denoted as PPO-BERTScore or GRPO-BERTScore) and vanilla LLMs, across representative LLMs and standard evaluation datasets.
\vspace{-10pt}
\paragraph{Experiment Setting.}

Training is performed on a single node with four A100 or H100 GPUs. Unless stated otherwise, we use a rollout group size of 6 per prompt, KL penalty of \(\lambda_{\mathrm{KL}}=0.001\) to stabilize optimization, and cap decoder outputs at $2048$ tokens and input within $512$ tokens. 
In determining $R_p$, extensive experiments indicate that a good configuration could be \(p=0.2\) and BERTScore, Jacobian scores weights, respectively, as \((w_0,w_1,w_2,w_3)=(0.7,0.1,0.1,0.1)\), which is set as default for comparison unless otherwise specified in ablation studies. We set learning rate of \(2\times10^{-6}\), batch size \(32\), and training for 10 epochs with temperature $0.6$. Validation interval is set as $10$ and checkpoints are saved per $100$ steps. Experiments are based on Verl~\citep{sheng2024hybridflow} framework with mild modification for separate thread for reward calculation, and we also set warm-up step = 10 for critic.

To avoid random guesses, all models are prompted to generate CoT reasoning \citep{Wei2022chain}. Evaluation is reported as accuracy, and the response correctness is determined via an LLM-as-Judge protocol \citep{zheng2023judging} using GPT-4o-mini (see prompt in \autoref{fig:binary_template}, which is designed to distinguish guess or contradictions especially 
for multi-choice questions).

We evaluate both standard instruction-tuned LLMs and “thinking” variants that emit explicit intermediate reasoning and instruction following. Although compute constraints preclude experiments on 70B-scale models, our evaluation covers compact to mid-sized systems that are widely used in research: Llama-3-8B-Instruct, Llama-3.2-3B-Instruct, Phi-3.5-mini-Instruct, Qwen3-4B-Thinking, and Qwen3-1.7B-Thinking  \citep{llama3_8b_ins_hf,llama3_2_3b_ins_hf,phi35mini_hf,qwen3_4b_thinking_hf,qwen3_1p7b_hf}. \footnote{Note that all experiments use Instruct or Thinking models, as our outputs must be cleanly formatted into $X$ and $Y$, which base models struggle to separate. Since Instruct models are typically harder to improve with RL, the observed gains further highlight the effectiveness of CE-PO.}


\noindent\textbf{Training and Validation Dataset.}\quad
We train and tune on four reasoning-heavy sets: \emph{BBEHCausal}, the causal reasoning split from BIG-Bench Extra Hard dataset, including counterfactual deduction or cause inference questions; \citep{BBEHGitHub}); \emph{CaseHOLD}, tasks centering on multiple-choice legal holdings given case citation context \citep{CaseHOLDHF}; \emph{MATHHARD}, the level-5 subset targeting multi-step mathematical reasoning with multiple-choice questions~\citep{hendrycksmath2021}; and \emph{IfQA}, an open-domain QA under counterfactual \emph{if}-clause presuppositions requiring hypothetical reasoning \citep{IfQAGitHub}. Note that all datasets are formatted as question answer pairs, examples of them are presented respectively in \autoref{tab:dataset-examples1} and \autoref{tab:dataset-examples2}.

\noindent\textbf{Testing Dataset.}\quad
For same-field evaluation, we use \emph{BBEHMATH} (challenging multistep arithmetic), where we filter shorter and more trivial questions and normalize alphabetic numerals to digits with unfolded calculation for comparability on this dataset~\citep{BBEHGitHub}. We further assess generalization with \emph{CLadder}, covering association/intervention/counterfactual queries \citep{CLadderHF}, \emph{LegalBench} (case-understanding split) \citep{guha2023legalbench} analyzing the cause or reason of legal case, and \emph{LogiQA}, a multiple-choice dataset emphasizing deductive logical reasoning sourcing from officer entrance exams~\citep{LogiQAHF}.



\vspace{-10pt}
\paragraph{Main Experiment Results.}

Under the training and validation dataset split detailed above, \autoref{fig:validation} shows CE-PPO delivering smooth, near-monotonic validation gains across five backbones and four validation sets: rapid early improvement followed by mild saturation, which is indicative of stable and standard RL optimization plot rather than overfitting. In \autoref{tab:mainresults}, across all five backbones, both causal variants outperform their non-causal baselines: CE-PPO/CE-GRPO improve the best non-CE baseline by 2.3 to 9.58 
points. Comparing CE-GRPO with the vanilla baseline, we observe gains of 5.8–6.0 points on Llama-3-8B and Llama-3.2-3B; the largest improvement appears on Qwen-3-1.7B-Thinking ($\approx 10$ points), whereas Phi-3.5-mini-Instruct shows the most marginal uplift. These patterns suggest that architectural and scale differences modulate the benefits of CE-GRPO. CE-GRPO is \emph{not} uniformly stronger than CE-PPO, while slightly higher on Qwen-1.7B/4B, Phi-3.5-mini, and Llama-3-8B, it lags behind CE-PPO on Llama-3.2-3B, indicating complementary strengths rather than strict dominance.

\begin{figure}
    \centering
    \includegraphics[width=\linewidth]{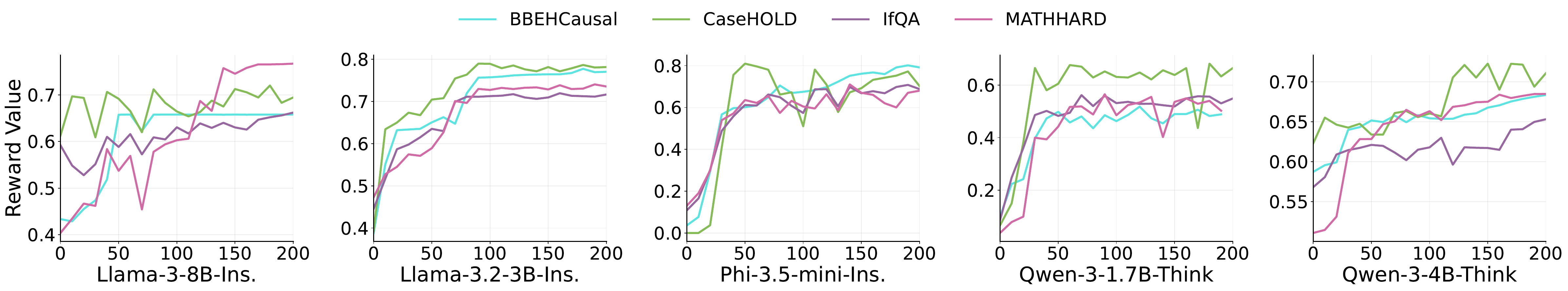}
    \vspace{-11pt}
    \caption{Reward curves of CE-PPO training on validation datasets across five models. }
    \vspace{-3pt}
    \label{fig:validation}
\end{figure}

\begin{table}[t!]
    \centering
    \small \vspace{-1pt}
    \caption{\textbf{Main results across models, RL methods, and datasets.} 
We report LLM-as-Judge accuracy (\%) 
under different RL variants: vanilla model , standard PPO/GRPO with BERTScore reward, and causal-enhanced (CE-PPO/CE-GRPO). We repeatly generate 4 times of each trained model, and thus report mean and deviation below.}
    \renewcommand{\arraystretch}{1}
    \setlength{\tabcolsep}{6pt}\vspace{-2pt}
    \resizebox{\linewidth}{!}{%
    \begin{tabular}{l l c c c c c}
\toprule[1pt]
\multirow{2}{*}{\textbf{Model}} 
  & \multirow{2}{*}{\textbf{RL Method}} 
  & \multicolumn{4}{c}{\textbf{Dataset (\%)}} 
  & \multirow{2}{*}{\textbf{Average}} \\
\cmidrule(lr){3-6}
  &  & \textbf{BBEHMATH} & \textbf{CLadder} & \textbf{LegalBench} & \textbf{LogiQA} & \\
\midrule

\rowcolor{gray!8}
\cellcolor{white}{\multirow{6}{*}[-1.5ex]{\centering\textbf{Qwen-3-1.7B-Thinking}}}
  & Vanilla            & 1.36 $\pm$ 1.96 & 34.77 $\pm$ 0.68 & 57.80 $\pm$ 2.59 & 50.00 $\pm$ 0.45 & 35.98\\
  & PPO-BERTScore   & 5.91 $\pm$ 0.00 & 37.27 $\pm$ 0.91 & 58.26 $\pm$ 1.95 & 52.50 $\pm$ 0.00 & 38.48\\
\rowcolor{gray!8}
\cellcolor{white}{}
  & GRPO-BERTScore  & 3.64 $\pm$ 1.29 & 35.00 $\pm$ 0.52 & 60.67 $\pm$ 0.94 & 57.04 $\pm$ 0.68 & 39.09\\
  & CE-PPO          & \textbf{10.45 $\pm$ 0.52} & \textbf{49.32 $\pm$ 0.68} & 60.67 $\pm$ 1.05 & \textbf{58.05 $\pm$ 1.14} & 44.62\\
\rowcolor{gray!8}
\cellcolor{white}{}
  & CE-GRPO         & 10.00 $\pm$ 1.29 & 44.77 $\pm$ 0.00 & \textbf{70.64 $\pm$ 0.00} & 56.82 $\pm$ 0.45 & \textbf{45.56}\\
\midrule

\rowcolor{blue!5}
\cellcolor{white}{\multirow{6}{*}[-1.5ex]{\centering\textbf{Qwen-3-4B-Thinking}}}
  & Vanilla            & 6.82 $\pm$ 0.45 & 42.95 $\pm$ 0.91 & 74.77 $\pm$ 1.95 & 80.45 $\pm$ 0.68 & 51.25\\
  & PPO-BERTScore   & 5.00 $\pm$ 1.05 & 45.00 $\pm$ 0.00 & 73.64 $\pm$ 0.47 & 80.00 $\pm$ 0.45 & 50.91\\ 
\rowcolor{blue!5}
\cellcolor{white}{}
  & GRPO-BERTScore  & 6.82 $\pm$ 0.45 & 40.68 $\pm$ 0.68 & 76.61 $\pm$ 0.65 & 83.86 $\pm$ 0.91 & 51.99\\
  & CE-PPO          & \textbf{9.55 $\pm$ 0.52} & 45.45 $\pm$ 0.68 & \textbf{79.34 $\pm$ 1.05} & 79.77 $\pm$ 0.91 & 53.53\\ 
\rowcolor{blue!5}
\cellcolor{white}{}
  & CE-GRPO         & 7.27 $\pm$ 0.00 & \textbf{49.09 $\pm$ 0.64} & 76.90 $\pm$ 0.00 & \textbf{84.09 $\pm$ 0.45} & \textbf{54.34}\\
\midrule

\rowcolor{gray!8}
\cellcolor{white}{\multirow{6}{*}[-1.5ex]{\centering\textbf{Phi-3.5-mini-Ins.}}}
  & Vanilla            & 6.13 $\pm$ 0.52 & 40.23 $\pm$ 0.68 & 68.35 $\pm$ 1.95 & 54.09 $\pm$ 0.45 & 42.20\\
  & PPO-BERTScore   & 6.59 $\pm$ 0.00 & 43.40 $\pm$ 0.87 & 66.97 $\pm$ 0.00 & 54.77 $\pm$ 0.00 & 42.93\\
\rowcolor{gray!8}
\cellcolor{white}{}
  & GRPO-BERTScore  & 5.45 $\pm$ 0.87 & 41.59 $\pm$ 0.64 & 68.81 $\pm$ 1.30 & 52.73 $\pm$ 0.45 & 42.15\\
  & CE-PPO          & \textbf{8.41 $\pm$ 0.00} & 42.95 $\pm$ 0.91 & \textbf{70.18 $\pm$ 1.95} & 55.90 $\pm$ 0.68 & 44.36\\
\rowcolor{gray!8}
\cellcolor{white}{}
  & CE-GRPO         & 6.81 $\pm$ 0.64 & \textbf{45.45 $\pm$ 0.45} & \textbf{70.18 $\pm$ 0.65} & \textbf{57.95 $\pm$ 0.00} & \textbf{45.10}\\
\midrule

\rowcolor{blue!5}
\cellcolor{white}{\multirow{6}{*}[-1.5ex]{\centering\textbf{Llama-3.2-3B-Ins.}}}
  & Vanilla            & 3.63 $\pm$ 0.91 & 30.00 $\pm$ 0.68 & 45.41 $\pm$ 0.47 & 41.59 $\pm$ 0.45 & 30.16\\
  & PPO-BERTScore   & 6.59 $\pm$ 0.68 & 33.63 $\pm$ 0.64 & 45.87 $\pm$ 0.00 & 38.41 $\pm$ 0.45 & 31.12\\
\rowcolor{blue!5}
\cellcolor{white}{}
  & GRPO-BERTScore  & 6.50 $\pm$ 0.45 & 31.36 $\pm$ 0.91 & 42.73 $\pm$ 2.57 & 44.09 $\pm$ 0.00 & 31.17\\
  & CE-PPO          & \textbf{10.67 $\pm$ 1.14} & \textbf{43.64 $\pm$ 0.68} & \textbf{55.05 $\pm$ 1.30} & \textbf{48.18 $\pm$ 0.91} & \textbf{39.38}\\
\rowcolor{blue!5}
\cellcolor{white}{}
  & CE-GRPO         & 7.95 $\pm$ 0.91 & 35.23 $\pm$ 0.00 & 54.13 $\pm$ 2.59 & 47.27 $\pm$ 0.45 & 36.15\\
\midrule

\rowcolor{gray!8}
\cellcolor{white}{\multirow{6}{*}[-1.5ex]{\centering\textbf{Llama-3-8B-Ins.}}}
  & Vanilla           & 7.05 $\pm$ 0.45 & 39.54 $\pm$ 0.91 & 61.93 $\pm$ 1.95 & 65.68 $\pm$ 0.64 & 43.55\\
  & PPO-BERTScore   & 5.27 $\pm$ 0.64 & 37.95 $\pm$ 0.00 & 61.93 $\pm$ 0.00 & 60.91 $\pm$ 0.45 & 41.52\\
\rowcolor{gray!8}
\cellcolor{white}{}
  & GRPO-BERTScore  & 10.91 $\pm$ 1.14 & 39.09 $\pm$ 0.91 & 57.80 $\pm$ 0.00 & 67.50 $\pm$ 0.00 & 43.83\\
  & CE-PPO          & \textbf{14.86 $\pm$ 1.59} & 44.77 $\pm$ 0.45 & \textbf{67.35 $\pm$ 0.79} & 68.40 $\pm$ 0.91 & 48.84\\ 
\rowcolor{gray!8}
\cellcolor{white}{}
  & CE-GRPO         & 13.64 $\pm$ 0.91 & \textbf{45.00 $\pm$ 0.64} & 66.97 $\pm$ 0.00 & \textbf{71.82 $\pm$ 0.45} & \textbf{49.36}\\
\bottomrule[1pt]
\end{tabular}
    }
    \label{tab:mainresults}
    \vspace{-10pt}
\end{table}

\vspace{-10pt}
\paragraph{Out of Distribution(OOD) Evaluation.}Trained exclusively on causal- and math-reasoning corpora, our CE-PO models are stress-tested for OOD generalization on TruthfulQA \citep{lin2021truthfulqa} (misconception resistant factuality), CodeMMLU \citep{manh2024codemmlu} (programming knowledge), and SuperGPQA \citep{du2025supergpqa} (expert level STEM QA), evaluating 200 uniformly random samples from each benchmark. As shown in \autoref{fig:ood_plots}, they deliver consistent—but task-dependent improvements over baselines: effects are strongest on structure-heavy code reasoning, moderate on STEM-oriented SuperGPQA, and smaller yet steady on TruthfulQA, suggesting that causal regularization chiefly benefits multi-step reasoning while preserving overall robustness.

\begin{figure*}[h]
  \centering
  \includegraphics[width=0.9\textwidth]{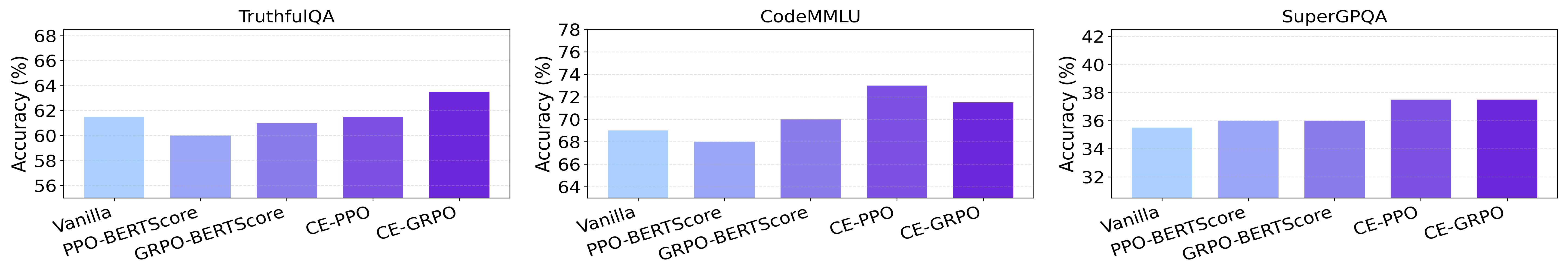} 
  \vspace{0pt}
  \caption{Qwen-3-4B-Thinking accuracy in OOD evaluation scenarios.} 
  \vspace{0pt}
  \label{fig:ood_plots}
\end{figure*}

\vspace{-10pt}
\paragraph{Ablation Studies.}
\label{subsec:ablation}
We test CE-PO's composite objective under a fixed setup (Llama-3-8B-Instruct, GRPO, max\_tokens=2048; datasets: BBEHMath, CLadder, LegalBench, LogiQA). The base reward is $r_{\text{base}}$ (BERTScore); coherence uses hardened link scores $S(Z,X)$, $S(X,Y)$, $S(Z,Y)$. We ablate the Minkowski exponent $p$ and weights $w$, drop individual links, add Gaussian noise to $r_{\text{base}}$, and swap BERTScore for an LLM-as-Judge. Ablation results in \autoref{tab:ablation_final} yield:

\textbf{\ul{Insight 1:}} Reward hacking would emerge once we remove any of $S(Z,X)$, $S(Z,Y)$ and $S(X,Y)$. Models under such flawed RL can't even excel the original models as observed in row 5 to 7 in \autoref{tab:ablation_final}. The likely cause is an incomplete constraint, i.e., the causal loop \(Z \rightarrow X \rightarrow Y\) is not closed, so the optimization ceases to be strictly benign.  With a detailed supporting plot provided in \autoref{fig:ABLATIONSUPPORT} of the reward signal and generation length, we can observe that initial removal breaks the prior balance (reward dips), and subsequent updates discover a length shortcut (length conflating).

\textbf{\ul{Insight 2:}} By selecting different $p$ values, we can observe the performances of $p$ as 10 identical to BERTScore only and $p$ as -2 to Jacobian scores only (i.e. $w$ as [0,1/3, 1/3, 1/3]). The same case of $p$ as 1, which reveals that tuning $p$ can dynamically tune the optimization objective.

\textbf{\ul{Insight 3:}} Under a perturbed reward (Gaussian noise $\sigma$=0.05), the model maintains satisfactory performance, indicating robustness of the training design. Using GPT - 4o - mini as an LLM-as-Judge (vs. BERTScore) yields performance similar to the original model, revealing what's prescribed in the \emph{Note of $r_{base}$} in \autoref{subsec: jacob1}.

\begin{table}[t!] \vspace{-0.05in}
\centering \small
\caption{Ablation Studies on Llama-3-8B-Ins.}
\rowcolors{2}{blue!3}{white} 
\resizebox{0.77\linewidth}{!}{
\begin{tabular}{lccccc}
\toprule[1pt]
\textbf{Ablation} & \textbf{BBEHMATH} & \textbf{CLadder} & \textbf{LegalBench} & \textbf{LogiQA} & \textbf{Average}\\
\midrule
$p$ = 10     &6.36 &39.09 &56.90  &67.27 & 42.91\\
$p$ = -2     &12.73 &43.63&65.52  &70.00 & 47.97\\
$p$ = 1 &10.00 &50.00 &60.34  &64.55 & 46.22\\
$w$=[0, 1/3, 1/3, 1/3] &10.91 &44.54 &61.93  &68.18 & 46.89\\
$w$=[4/5, 1/10, 1/10, 0]   & 5.45&23.63 & 40.37 &41.82 & 27.82\\
$w$=[4/5, 1/10, 0, 1/10]   &3.63 &30.00 & 45.87 &49.09 & 32.15\\
$w$=[4/5, 0, 1/10, 1/10]   &5.45 &31.82 &43.12 & 40.90& 30.32\\
$w$=[1/2, 1/6, 1/6, 1/6]   &8.18 &41.82 &67.24 &66.36 & 45.90\\
Perturbed Reward &7.27 &40.00 &65.52  &69.09 & 45.97\\
Judge Reward (binary) &6.36 &39.09 &60.34 &66.36 & 43.04\\
\midrule
\rowcolor{green!3} \textbf{CE-GRPO}  &13.64 & 47.27&66.97 &72.73 & 50.65\\
\bottomrule[1pt]
\end{tabular}}
\label{tab:ablation_final}
\vspace{-13pt}
\end{table}


\paragraph{Generation Pattern Analysis.}


\begin{figure}[h]
  \centering
  \vspace{-8pt}
  \includegraphics[width=0.95\linewidth]{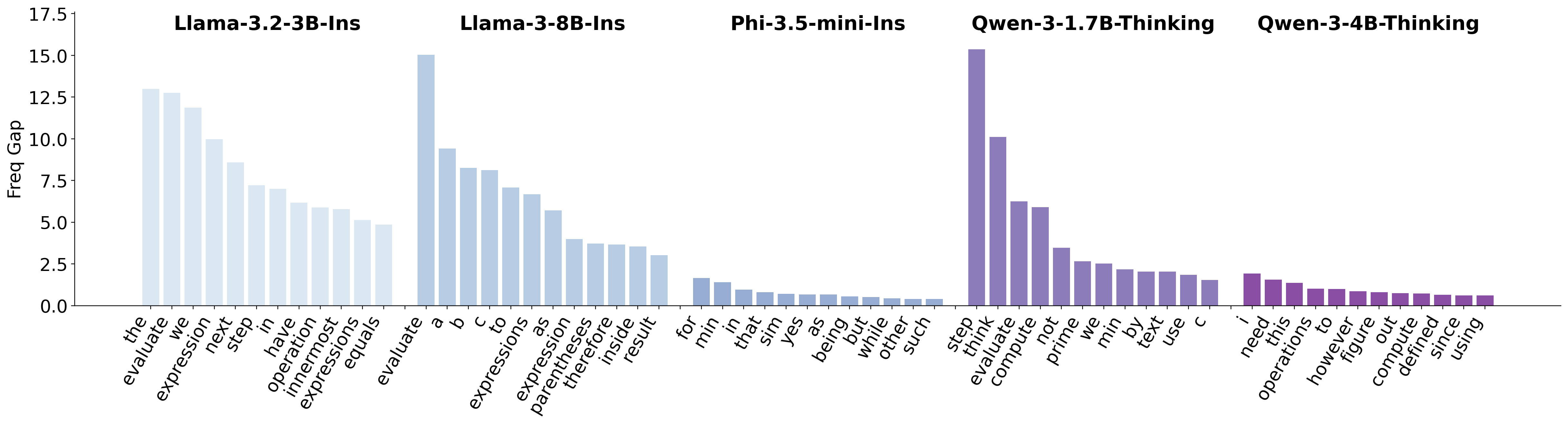}\vspace{-2pt}
  \caption{Frequency gap of generation per word where CE-GRPO exceeds BERTScore and Vanilla LLM (Top-12 per model). }
  \vspace{-8pt}
  \label{fig:ce_over_words}
\end{figure}

\begin{figure*}[t]
  \centering
  \includegraphics[width=\linewidth]{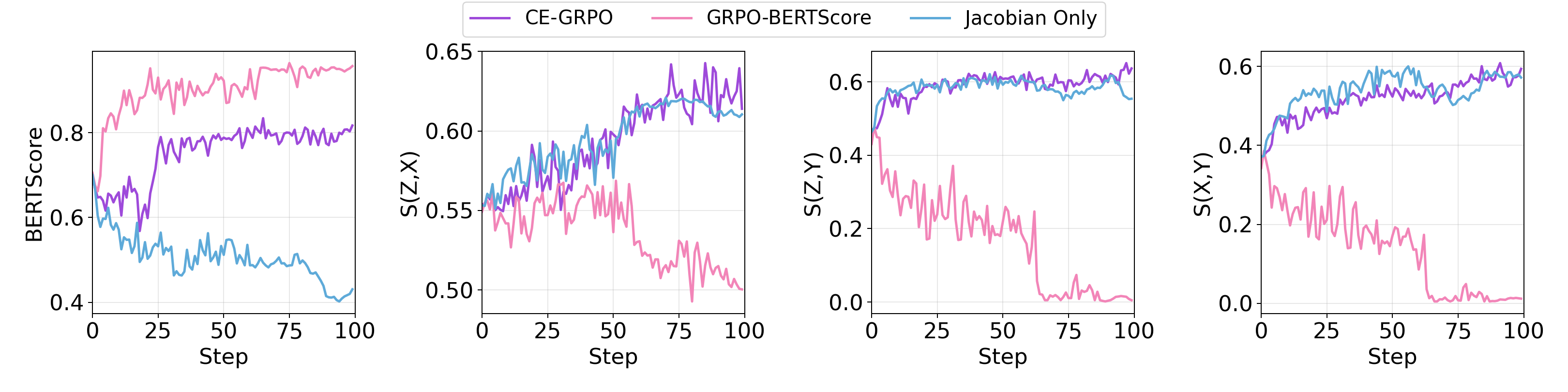}
  \vspace{-20pt}
  \caption{Training trajectories on Llama-3-8B-Instruct, reporting
   \SB, \SZX, \SZY, \SXY (left to right). Comparison is made between CE-GRPO and optimization with only causal signal (Jacobian) or accuracy signal (BERTScore), supporting the claim  of reward signal balance in \autoref{subsec:RLTradeoff}.   
   }
  \vspace{-15pt}
  \label{fig:curves-llama3-8b}
\end{figure*}

\autoref{fig:ce_over_words} shows that CE-GRPO does not induce uniform gains across all tokens but selectively amplifies reasoning or causal related words, suggesting that causal regularization guides the model toward tokens central to stepwise inference rather than surface templates. Complementarily, \autoref{fig:causal_pairplot} reveals the structural relation between metrics: $S(Z,X)$ is nearly independent of $S(X,Y)$ and $S(Z,Y)$, while the latter two are moderately correlated. This indicates that “evidence infusion” and “answer stability” provide orthogonal yet complementary signals. Their joint use, as in CE-GRPO, therefore supports coherence across the full $Z \!\to\! X \!\to\! Y$ pathway, rather than relying on any single proxy.

\begin{wrapfigure}[7]{r}{0.4\textwidth} 
    \centering
    \vspace{-32pt}
    \includegraphics[width=0.4\textwidth]{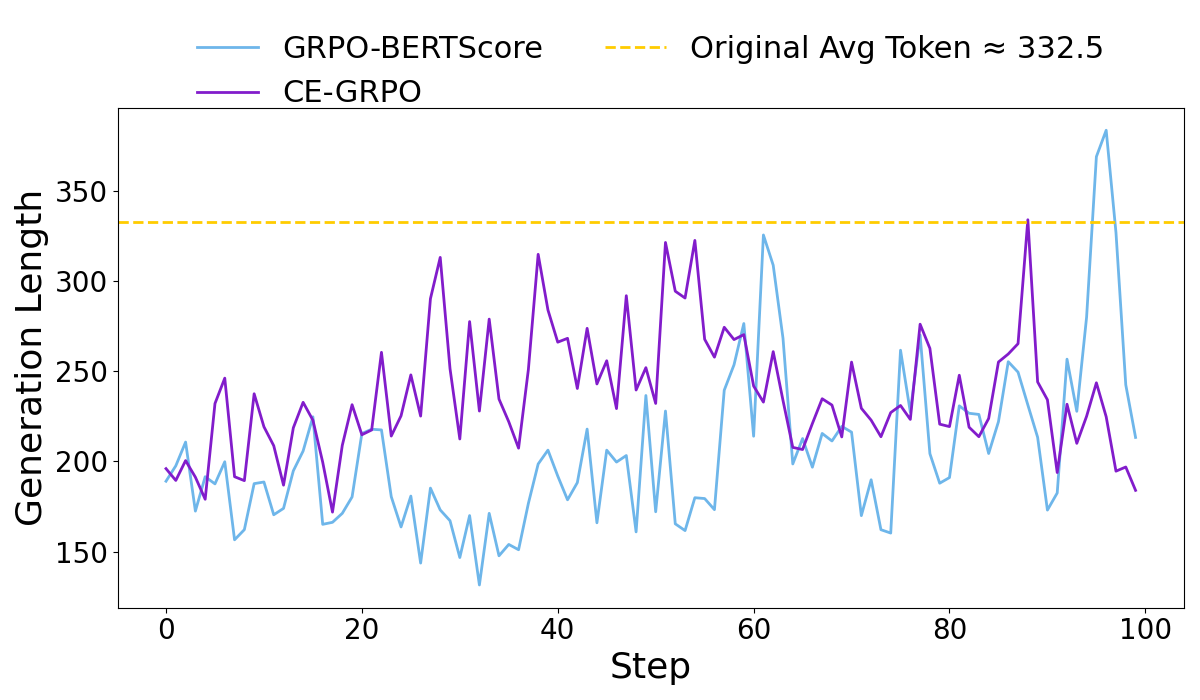}
    \vspace{-14pt}
    \caption{Response length comparison on Llama-3-8B-Ins}
    \label{fig:llama3-8b-response-length}
\end{wrapfigure}

Turning to sequence-level behavior, \autoref{fig:llama3-8b-response-length} reports the response lengths of CE-GRPO, GRPO-BERTScore, and the vanilla model on the training sets.
The means are closely matched, indicating our optimization is largely \emph{length-indifferent}.
CE-GRPO achieves better task performance without increasing the token budget beyond the baseline, suggesting the efficiency of CE-GRPO.

\begin{wrapfigure}[14]{r}{0.40\textwidth} 
    \centering
    \vspace{10pt} 
    \includegraphics[width=0.85\linewidth]{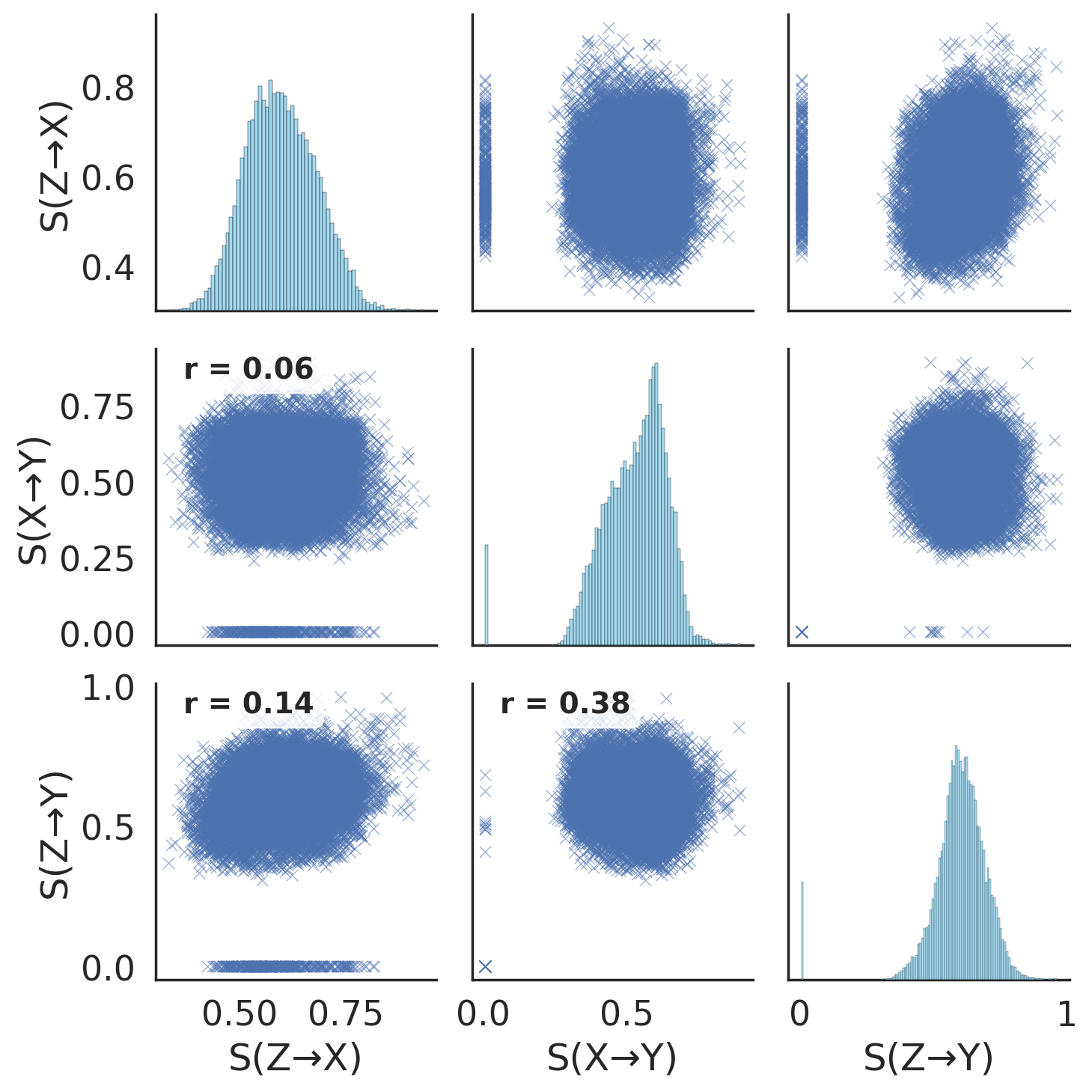}
    \vspace{5pt}
    \caption{Pairwise scatter matrix of causal metrics on Qwen-3-4B-Thinking.}
    \label{fig:causal_pairplot}
\end{wrapfigure}
Finally, the training reward trajectories in \autoref{fig:curves-llama3-8b} highlight three regimes: BERTScore-only optimization inflates \SB\ only while collapsing causal channels, Jacobian-only boosts sensitivities but remains unstable and harms task accuracy. This finding echoes prior work showing that proxy metric optimization (e.g., ROUGE) misaligns with human preferences~\citep{Stiennon2020}, while self-rewarded training inflates output length, underscoring the tension between verifiable rewards and reward-gaming behaviors~\citep{Yuan2025SRLM}. CE-GRPO achieves smooth improvements across $S(Z,X)$, $S(X,Y)$, and $S(Z,Y)$ while keeping \SB\ competitive. This pattern underscores CE-GRPO’s ability to sustain the causal chain from prompt to reasoning to answer, offering a practical safeguard against shortcut learning.

\vspace{-0.1in}
\begin{wrapfigure}[10]{l}{0.40\textwidth}
    \centering
    \vspace{-33pt}
    \includegraphics[width=0.98\linewidth]{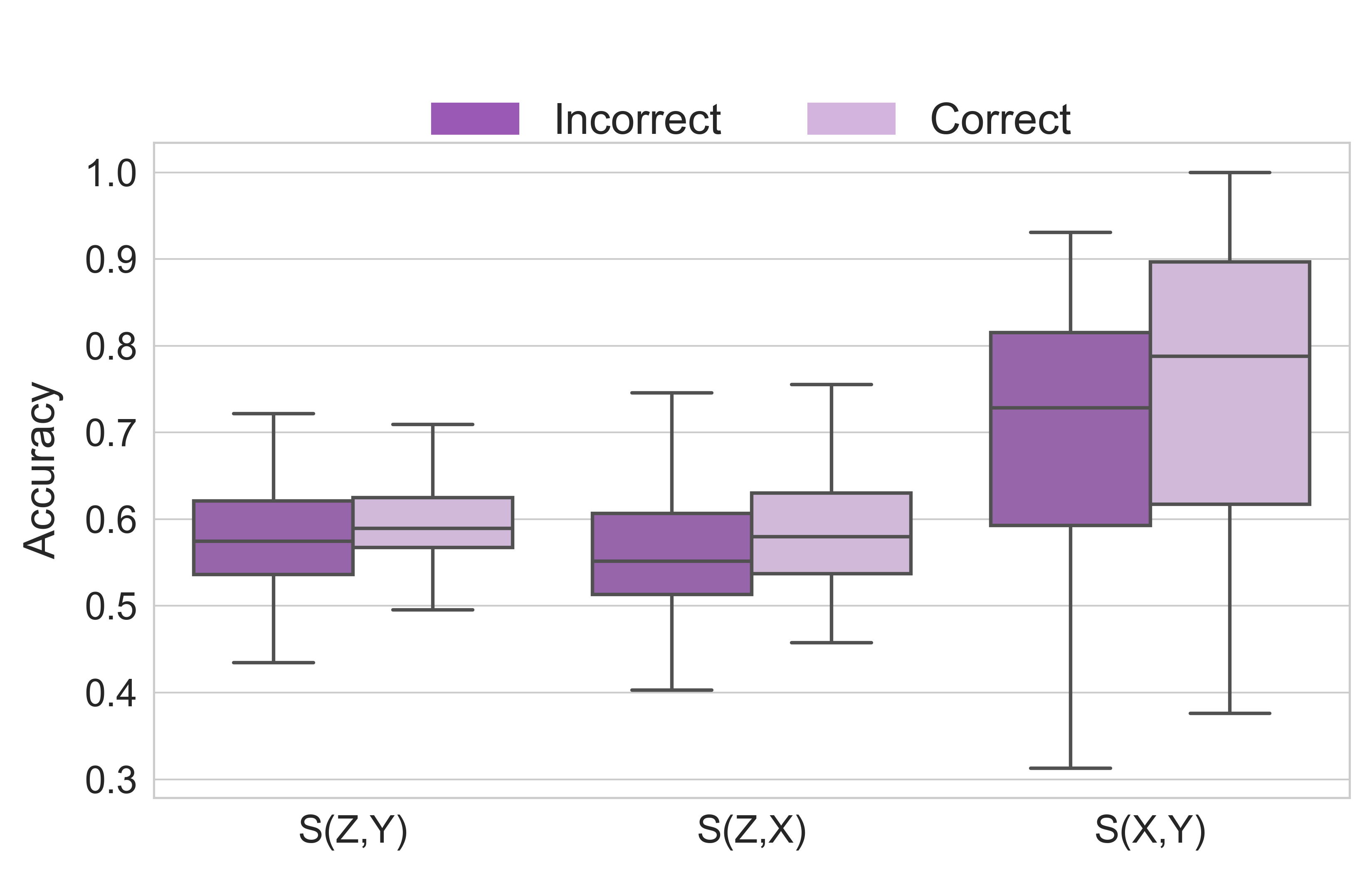}
    \caption{Boxplots of $S(Z,X)$, $S(Z,Y)$, and $S(X,Y)$ for LLM-as-Judge accuracy regarding their causal trajectory.}
    \label{fig:corr-boxplot}
    \vspace{-10pt}
\end{wrapfigure}

\paragraph{Jacobian Signals and Accuracy} 
\autoref{fig:corr-boxplot} shows the distributions of Jacobian-based scores for samples annotated by correctness with LLM-as-Judge with strict judging criteria provided in \autoref{fig:binary_template} conducted by GPT-4o. Annotations were collected on 1000 samples generated by Llama-3-8B-Ins. In total, 371 cases were labeled as consistent and 629 as inconsistent. To assess differences between the two groups, we conducted t-test with results reported in \autoref{tab:sabtest}. We evaluate whether higher $S(A,B)$ scores align with factual correctness, as validity of these causal signals requires separation under strict accuracy criteria. As in \autoref{tab:sabtest}, all three metrics show significantly higher means for the correct group ($p<0.01$), confirming that $S(A,B)$ faithfully tracks reasoning–answer validity beyond surface outputs, and by statistically aligns with cognitive accuracy.

\vspace{-0.1in}
\section{Conclusion} \vspace{-0.1in}
We propose counterfactual-enhanced policy optimization that couples task accuracy with continuous causal signals, yielding smoother training and consistent gains across models and datasets, especially on Thinking backbones. The approach is a drop-in for PPO/GRPO and reduces reward-hacking behaviors, improving stability during training and robustness at convergence. Future work based on CE-PO could possible work on scale interventions and automate the accuracy–causality balance; CE-PO opens new optimization avenues by incorporating richer causal signals (e.g., Hessian-based beyond simple Jacobian) and enabling adaptive schemes that tune this trade-off on the fly.

\newpage
\bibliography{iclr2025_conference}
\bibliographystyle{plainnat}

\appendix

\newpage
\section{Related Works}

\subsection{Reward Hacking in LLM}

Reward hacking occurs when models exploit imperfections in reward functions rather than achieving the intended objective~\citep{amodei2016concrete, everitt2017reinforcement, zhou2025evolving}. In large language models, this problem emerges prominently under alignment: proxy reward models can be gamed, leading to behaviors such as generating superficially persuasive but incorrect answers~\cite{wen2024language}, sycophancy~\cite{shrama2023towards}, and exploiting evaluator biases in LLM-as-judge settings~\cite{liu2023llms,wang2023large}. More recently, iterative refinement has revealed \emph{in-context reward hacking}, where models adapt outputs to maximize evaluation scores rather than genuinely improving reasoning quality~\cite{pan2023spontaneous,pan2024feedback}. Mitigation strategies typically involve improved supervision (e.g., rationale-level feedback~\cite{ross2017right}), evaluator debiasing, or structural interventions such as counterfactual augmentation~\cite{kaushik2020explaining} and invariance penalties~\cite{Arjovsky2019invariant}. However, these often rely on external annotations nor emphasized robustness–accuracy trade-offs. Our method introduces differentiable, counterfactual-enhanced causal signals as internal rewards, jointly optimizing accuracy and causal coherence to directly curb shortcut exploitation.


\subsection{Causal Inference and Reasoning in LLMs} 

There have been several studies centering on causal reasoning capability in LLMs. Corr2Cause finds poor correlation-to-causation transfer by pure LLM reasoning \citep{Jin2024corr2cause}, CoT often fails to be causally responsible for answers \citep{Bao2024NonCausalCoT}, and faithful CoT remains challenging even with fine-tuning or activation edits \citep{Tanneru2024CoTFaithfulnessHardness}. Methods that \emph{impose} causal structure are emerging (e.g., G$^2$-Reasoner combines knowledge and goal prompts to improve causal reasoning) \citep{Chi2025G2Reasoner}. Mechanistic approaches intervene directly in internal representations—editing factual circuits and testing hypotheses via causal/activation patching—to probe and modify pathways \citep{Heimersheim2024ActivationPatching}. There also several studies utilizing SCM~\citep{yu2025towards} and causal graph~\citep{sheth2025causalgraph2llm} for better reasoning. Yet most studies are \emph{post-hoc}. We build on this direction by using interventional, pathway-aware signals during post-training, and by explicitly aggregating them with task rewards to control the accuracy–causal-coherence trade-off.

\section{Theoretical Guarantees}
\label{sec:theory}
\paragraph{Setup.}
Let $\pi,\pi_{\mathrm{old}}$ be stationary policies, $\gamma\in(0,1)$ the discount factor,
and $d^{\pi}$ the normalized discounted state distribution.
Fix the baseline policy $\pi_{\mathrm{old}}$ and let $A(s,a)=A_{\pi_{\mathrm{old}}}(s,a)$
be its advantage under the CE-PO reward $r_{\mathrm{CE}}$ (the Minkowski combiner of all channels).
For each state $s$, denote by $\mathcal{N}_s\subset\mathbb{R}^{|\mathcal{A}|}$ 
the subspace spanned by \emph{nuisance directions} (e.g., verbosity/length) estimated 
via counterfactual perturbations, and let $P_s$ be the Euclidean orthogonal projector 
onto $\mathcal{N}_s$.
Write the action simplex
$\Delta^{|\mathcal{A}|}=\{p\in\mathbb{R}^{|\mathcal{A}|}:p\ge 0,\;\mathbf{1}^\top p=1\}$ with
tangent space $\mathcal{T}_s=\{v\in\mathbb{R}^{|\mathcal{A}|}:\mathbf{1}^\top v=0\}$.
Define the \emph{orthogonalized advantage}
\[
A_{\perp}(s,\cdot) := (I-P_s)\,A(s,\cdot),\qquad
\varepsilon_{\perp} := \max_{s}\,\|A_{\perp}(s,\cdot)\|_{\infty}.
\]
and let
\[
\alpha := \max_{s} D_{\mathrm{TV}}\!\bigl(\pi(\cdot|s),\pi_{\mathrm{old}}(\cdot|s)\bigr)
=\tfrac12\max_s\|\pi(\cdot|s)-\pi_{\mathrm{old}}(\cdot|s)\|_1.
\]
The TRPO-style surrogate is
\[
L_{\pi_{\mathrm{old}}}^{\mathrm{CE}}(\pi)
=\eta_{\mathrm{CE}}(\pi_{\mathrm{old}})+\frac{1}{1-\gamma}\,
\mathbb{E}_{s\sim d^{\pi_{\mathrm{old}}}}
\bigl[\;\langle A(s,\cdot),\,\pi(\cdot|s)-\pi_{\mathrm{old}}(\cdot|s)\rangle\;\bigr],
\]
noting that $\langle A(s,\cdot),\pi_{\mathrm{old}}(\cdot|s)\rangle=0$ for all $s$.

\paragraph{Assumptions and Propositions}
\begin{itemize}
\item[\textbf{P0}] \textbf{(Fixed reward at $\pi_{\mathrm{old}}$)}\;
All CE components used to construct $r_{\mathrm{CE}}$ (Jacobians, counterfactual
subspaces $\mathcal{N}_s$, projectors $P_s$, and channel scores) are computed at
$\pi_{\mathrm{old}}$ and held fixed while evaluating $\eta_{\mathrm{CE}}(\pi)$ over one update step.

\item[\textbf{A0}] \textbf{(Residualization on the action-simplex tangent)}\;
By construction of $r_{\mathrm{CE}}$ via counterfactual projection, the CE advantage is locally
invariant along nuisance directions at $\pi_{\mathrm{old}}$:
\[
\langle A(s,\cdot),\,v\rangle \;=\; 0 \quad \text{for all } v\in \mathcal{N}_s\cap\mathcal{T}_s.
\]
Equivalently, $P_sA(s,\cdot)=0$ and hence $A(s,\cdot)=A_{\perp}(s,\cdot)\in \mathcal{N}_s^\perp$.

\item[\textbf{A1}] \textbf{(Well-posedness \& bounded rewards)}\;
Base rewards and channel scores are in $[0,1]$ and CE weights sum to one, so $r_{\mathrm{CE}}\in[0,1]$.
The MDP is standard (finite or $\sigma$-finite), ensuring existence of $V^\pi,Q^\pi,A_\pi$.
\end{itemize}

\begin{theorem}[Orthogonalized TRPO Lower Bound for CE-PO]
For any $\pi$, with $\alpha=\max_s D_{\mathrm{TV}}\bigl(\pi(\cdot|s),\pi_{\mathrm{old}}(\cdot|s)\bigr)$,  
the CE-PO performance satisfies
\begin{equation}\label{eq:main}
\eta_{\mathrm{CE}}(\pi)\;\ge\;
L_{\pi_{\mathrm{old}}}^{\mathrm{CE}}(\pi)\;-\;
\frac{4\gamma}{(1-\gamma)^2}\;\varepsilon_{\perp}\;\alpha^2.
\end{equation}
\end{theorem}

\begin{proof}
\textbf{Step 1: Performance-difference decomposition.}
The performance-difference lemma gives
\[
\eta_{\mathrm{CE}}(\pi)-\eta_{\mathrm{CE}}(\pi_{\mathrm{old}})
=\frac{1}{1-\gamma}\,
\mathbb{E}_{s\sim d^{\pi}}\,\bigl[\langle A(s,\cdot),\,\pi(\cdot|s)\rangle\bigr].
\]
Add and subtract the surrogate:
\[
\eta_{\mathrm{CE}}(\pi)
=L_{\pi_{\mathrm{old}}}^{\mathrm{CE}}(\pi)
+\frac{1}{1-\gamma}\,
\mathbb{E}_{s}\Big[(d^{\pi}-d^{\pi_{\mathrm{old}}})(s)\,
\langle A(s,\cdot),\,\pi(\cdot|s)\rangle\Big]
=:L_{\pi_{\mathrm{old}}}^{\mathrm{CE}}(\pi)-\frac{1}{1-\gamma}T,
\]
so it suffices to bound $|T|$.

\textbf{Step 2: Bounding the occupancy shift.}
By a standard coupling argument,
$\|d^{\pi}-d^{\pi_{\mathrm{old}}}\|_{1}\le \tfrac{2\gamma}{1-\gamma}\,\alpha$.
Thus
\[
|T|
\le \big(\max_{s}|\langle A(s,\cdot),\,\pi(\cdot|s)\rangle|\big)\;
\|d^{\pi}-d^{\pi_{\mathrm{old}}}\|_{1}
\le \frac{2\gamma}{1-\gamma}\,\alpha\;
\max_{s}|\langle A(s,\cdot),\,\pi(\cdot|s)\rangle|.
\]

\textbf{Step 3: Using CE-PO orthogonalization (A1).}
For each $s$, by A1 we have $A(s,\cdot)=A_{\perp}(s,\cdot)\in \mathcal{N}_s^\perp$, and
$\langle A(s,\cdot),\pi_{\mathrm{old}}(\cdot|s)\rangle=0$ (definition of advantage). Hence
\[
\langle A(s,\cdot),\,\pi(\cdot|s)\rangle
=\langle A_{\perp}(s,\cdot),\,\pi(\cdot|s)-\pi_{\mathrm{old}}(\cdot|s)\rangle
\le \|A_{\perp}(s,\cdot)\|_{\infty}\,\|\pi(\cdot|s)-\pi_{\mathrm{old}}(\cdot|s)\|_{1}
\le 2\,\varepsilon_{\perp}\,\alpha.
\]
Taking the supremum over $s$ yields
$\max_{s}|\langle A(s,\cdot),\,\pi(\cdot|s)\rangle|\le 2\,\varepsilon_{\perp}\,\alpha$, and therefore
\[
|T|\le \frac{2\gamma}{1-\gamma}\,\alpha\,(2\,\varepsilon_{\perp}\,\alpha)
=\frac{4\gamma}{1-\gamma}\,\varepsilon_{\perp}\,\alpha^2.
\]

\textbf{Step 4: Final bound.}
Plugging into Step 1,
\[
\eta_{\mathrm{CE}}(\pi)\ge
L_{\pi_{\mathrm{old}}}^{\mathrm{CE}}(\pi)-\frac{4\gamma}{(1-\gamma)^2}\,\varepsilon_{\perp}\,\alpha^2,
\]
which is inequality (\ref{eq:main}).
\end{proof}

\paragraph{Properties.}
\emph{(Tightening via nuisance removal, $\ell_2$ statement).}
Let $\varepsilon_{\mathrm{base}}^{(2)}:=\max_s\|A_{\mathrm{base}}(s,\cdot)\|_{2}$ be the corresponding constant under a non-residualized baseline reward, and define
$\varepsilon_{\perp}^{(2)}:=\max_s\|A_{\perp}(s,\cdot)\|_{2}$.
If there exists $\rho\in(0,1]$ such that
$\|P_s A_{\mathrm{base}}(s,\cdot)\|_{2}\ge \rho\,\|A_{\mathrm{base}}(s,\cdot)\|_{2}$ for all $s$
(\emph{nontrivial nuisance energy}), then by Pythagoras for orthogonal projectors,
\[
\varepsilon_{\perp}^{(2)} \;\le\; \sqrt{1-\rho^2}\;\varepsilon_{\mathrm{base}}^{(2)}.
\]
Since $\|\cdot\|_{\infty}\le \|\cdot\|_{2}$, we also have
$\varepsilon_{\perp}\le \varepsilon_{\perp}^{(2)}$, yielding a strictly smaller penalty constant in the $\ell_2$ analogue of~\eqref{eq:main} (and a looser but still improved constant in $\ell_\infty$ up to norm equivalence).

\emph{(Null penalty for hacking-only updates).}
If for some $s$ the per-state policy change lies entirely in $\mathcal{N}_s$, i.e.,
$\pi(\cdot|s)-\pi_{\mathrm{old}}(\cdot|s)\in \mathcal{N}_s$, then
$(I-P_s)(\pi-\pi_{\mathrm{old}})=0$ and hence
$\langle A_{\perp}(s,\cdot),\,\pi(\cdot|s)-\pi_{\mathrm{old}}(\cdot|s)\rangle=0$,
so that state contributes nothing to the penalty term in~\eqref{eq:main}.

\paragraph{KL-based corollaries.}
By Pinsker’s inequality $D_{\mathrm{TV}}(p,q)\le \sqrt{\tfrac{1}{2}\mathrm{KL}(p\Vert q)}$,
if $\delta:=\max_s \mathrm{KL}\bigl(\pi_{\mathrm{old}}(\cdot|s)\Vert\pi(\cdot|s)\bigr)$, then
\begin{equation}\label{eq:kl-cor}
\eta_{\mathrm{CE}}(\pi)\;\ge\;
L_{\pi_{\mathrm{old}}}^{\mathrm{CE}}(\pi)\;-\;
\frac{2\gamma}{(1-\gamma)^2}\,\varepsilon_{\perp}\,\delta.
\end{equation}
If only an \emph{expected} per-state KL is enforced,
$\mathbb{E}_{s\sim d^{\pi_{\mathrm{old}}}}\big[\mathrm{KL}(\pi_{\mathrm{old}}\|\pi)\big]\le \bar{\delta}$,
then~\eqref{eq:kl-cor} holds with $\delta=C\,\bar{\delta}$ for a method-dependent constant $C\ge 1$
(e.g., via ratio clipping or per-state caps used in TRPO/PPO).

\paragraph{From counterfactual Jacobians to action-space nuisance.}
Let $\mathcal{U}_s$ be a finite set of representation-space directions obtained from
counterfactual Jacobians (e.g., top singular vectors of counterfactual $J_{AB}$ blocks),
and let $g_s$ be a smooth channel-score functional (e.g., spectral energy share).
Write $G_s:=\nabla_{\pi(\cdot|s)}\, g_s$ for its Jacobian with respect to action probabilities.
Define
\[
\mathcal{N}_s\;:=\;\mathrm{span}\big\{\,G_s^\top u\;:\;u\in \mathcal{U}_s\,\big\}\;\cap\;\mathcal{T}_s.
\]
If each channel score is \emph{residualized} by projecting $J_{AB}$ onto $\mathcal{U}_s^\perp$,
then, at $\pi_{\mathrm{old}}$, the resulting CE reward satisfies
$\langle \nabla_{\pi} r_{\mathrm{CE}}(s,\cdot),\,v\rangle=0$ for all
$v\in \mathcal{N}_s\cap\mathcal{T}_s$, implying $P_sA(s,\cdot)=0$ (Assumption~\textbf{A0}).

\section{Additional Techniques}
\label{app:additionaltechnique}
\paragraph{Jacobian causal score is time and space-efficient in LLMs.}
We design the Jacobian-based causal reward to be \emph{practical at scale}: it reuses forward activations, avoids any explicit Jacobian/Hessian materialization, and replaces costly subspace operations with a constant-time projection. Concretely, we rely on three tricks.
\emph{\textcolor{blue!50!black}{(i) Local Jacobian via vJP (vector--Jacobian Projection).}}
Sensitivities are computed with one forward and two backward passes using vector--Jacobian products (reverse-mode autodiff), i.e., for a test direction $U$ we use $\mathrm{vJP}(s_B,E(A);U)=\langle J_{AB},U\rangle_F=J_{AB}^\top\!\mathrm{vec}(U)$ and the scalar sensitivity $\|J_{AB}^\top\!\mathrm{vec}(U)\|_2$; so no dense $J\!\in\!\mathbb{R}^{Td\times Td}$ (or Hessian) is ever formed; working memory stays at $\mathcal{O}(Td)$ for token embeddings \citep{Baydin2018,Pearlmutter1994}.
\emph{\textcolor{blue!50!black}{(ii) Residual projection.}} The counterfactual direction is removed with a single inner product: if $g$ is the local Jacobian direction and $\hat{c}$ the (reshuffled) counterfactual direction, we use $r=g-(g^\top\hat{c})\hat{c}$, avoiding SVD/subspace construction and extra buffers \citep{GolubVanLoan2013}.
\emph{\textcolor{blue!50!black}{(iii) Mixed precision.}} Forward/backward in FP16/bfloat16 with dynamic loss scaling provides wall-clock gains while roughly halving activation memory, without changing the algorithmic interface \citep{Micikevicius2018,NVIDIAA100Whitepaper}.

\noindent\textit{\textbf{Throughput and footprint}.} On a single modern GPU (e.g., A100 80\,GB) and Llama-3-8B-Instruct, our implementation returns a reward in $\approx 2$\,s, with negligible persistent memory beyond base activations and the $\mathcal{O}(Td)$ embedding cache.

\section{Additional Plots}

\begin{figure}[h]
  \centering
  \begin{subfigure}{0.485\linewidth}
    \centering
    \includegraphics[width=\linewidth]{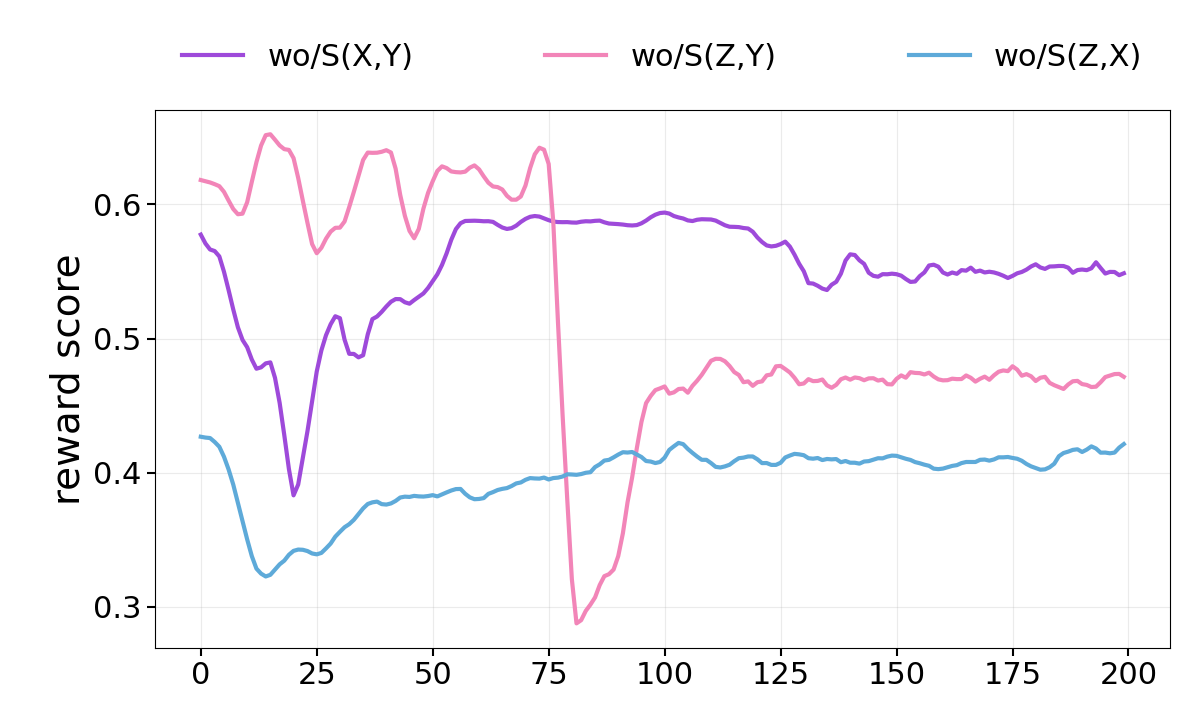} 
    \caption{Reward score}
    \label{fig:lengthmin}
  \end{subfigure}\hfill
  \begin{subfigure}{0.485\linewidth}
    \centering
    \includegraphics[width=\linewidth]{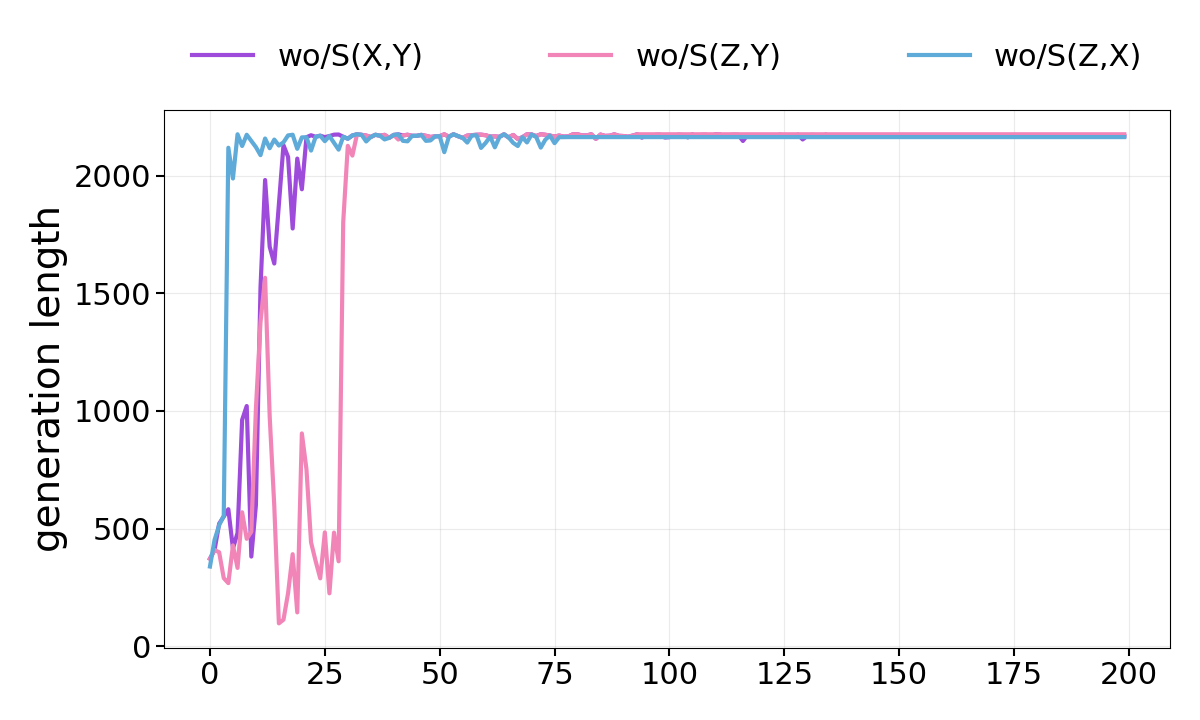} 
    \caption{Mean length}
    \label{fig:lengthmean}
  \end{subfigure}
  \vspace{-0.25em}
  \caption{This figure supplements the ablation analysis and demonstrates that omitting a single Jacobian term induces reward degradation accompanied by length inflation.}
  \vspace{-10pt}
  \label{fig:ABLATIONSUPPORT}
\end{figure}

\begin{tcolorbox}[
    colback=blue!5,
    colframe=black,
    coltitle=black,
    fonttitle=\bfseries,
    title=Example and Method Comparison,
    boxrule=0.6pt,
    sharp corners,
    breakable
]

\footnotesize
\textbf{Question.} \emph{Context: For husbands that don't set the alarm, the probability of ringing alarm is 90\%. For husbands that set the alarm, the probability of ringing alarm is 25\%. Question: For husbands that set the alarm, would it be less likely to see ringing alarm if the husband had not set the alarm?}

\medskip
\textbf{Ground Truth:} No

\medskip
\begin{center}
\setlength{\tabcolsep}{6pt}
\renewcommand{\arraystretch}{1.1}
\begin{tabular}{p{3cm} p{9cm}}
\toprule
\textbf{Method} & \textbf{Prediction (Reasoning)} \\
\midrule
BERTScore 
& Predicts ``Yes''; relies on surface overlap (\emph{...compares 25\% vs 90\% directly but ignores that the subject of ``less likely'' is ``not set''}). 
\\

Jacobian Only
& Predicts ``Yes''; observes 25\% $<$ 90\% but \emph{misplaces the subject, treating ``set'' as the focus of ``less likely'' instead of ``not set''}. 
\\

CE-GRPO 
& Predicts ``No''; correctly checks whether 90\% $<$ 25\% (\emph{false}) and concludes ``not set'' actually makes ringing more likely, matching ground truth. 
\\
\bottomrule
\end{tabular}
\end{center}

\medskip
\textbf{Takeaways.}
\begin{itemize}
\item \textbf{Jacobian Only} tends to \emph{template-follow}: latches onto ``less likely $\Rightarrow$ smaller \%'' and ignores subject/polarity flips, causing causal misread.
\item \textbf{BERTScore} measures overlap, not causal/conditional structure; it cannot reason about which condition the ``less likely'' modifies.
\end{itemize}
\label{tab:colorbox}

\end{tcolorbox}
\captionof{figure}{Comparison of methods on the alarm-setting counterfactual example on Qwen-3-4B-Thinking.}

\section{Broader Impact}

The proposed CE-PO framework has the potential to substantially improve the reliability and trustworthiness of large language models. By explicitly rewarding causal coherence in addition to task accuracy, the method reduces the tendency of models to exploit superficial cues or engage in reward hacking, which can lead to misleading or unfaithful reasoning. In practice, this enables LLMs to generate explanations and rationales that are not only correct in outcome but also grounded in meaningful reasoning processes—an essential property for high-stakes domains such as law, science, medicine, and education.

Beyond improving alignment, CE-PO can serve as a general methodology for building AI systems that are “right for the right reasons,” supporting transparency, auditability, and robustness under distribution shifts. This contributes to safer deployment of generative models in real-world decision-making pipelines. At the same time, emphasizing causal regularization may help mitigate harmful biases that arise when models rely on spurious correlations, thereby supporting more equitable AI systems.

\section{Prompt Templates}

We presented the prompt template for clear $Z$, $X$, $Y$ separation in \autoref{fig:finalanswer_instructions} and detailed instruction for LLM-as-Judge in \autoref{fig:binary_template}.

\begin{figure*}[h]
\begin{tcolorbox}[
  enhanced,
  colframe=blue!75!black,
  colback=white,
  coltitle=white,
  colbacktitle=blue!75!black,
  width=\linewidth,
  arc=2mm,
  auto outer arc,
  boxrule=0.5pt,
  left=10pt,
  right=10pt,
  drop shadow={black!50!white},
  top=10pt,
  bottom=10pt,
  title=\textbf{Prompt Instructions (Final Answer Tagging)},
  fonttitle=\bfseries,
  attach boxed title to top center={yshift=-2mm},
  boxed title style={sharp corners, size=small},
]
\small\ttfamily\raggedright
"Solve the question and generate in assigned format."\\
"Wrap your final answer between <finalanswer> </finalanswer> tags. "\\
"It's mandatory and required to not include anything after </finalanswer> tag. Don't let your generation process to redundant and lengthy."\\
"Be very clear in your explanation, and ensure the final answer is presented separately."\\
"Example of output: We analyze A, solve B and answer is C. So final answer is <finalanswer> C is correct </finalanswer>"
\end{tcolorbox}
\caption{Verbatim instructions for final-answer tagging.}
\label{fig:finalanswer_instructions}
\end{figure*}

\vspace{-10pt}
\begin{figure*}[h]
\begin{tcolorbox}[
  enhanced,
  colframe=brown!75!black,
  colback=white,
  coltitle=white,
  colbacktitle=brown!75!black,
  width=\linewidth,
  arc=2mm,
  auto outer arc,
  boxrule=0.5pt,
  left=10pt,
  right=10pt,
  drop shadow={black!50!white},
  top=10pt,
  bottom=10pt,
  title=\textbf{Prompt Template (Strict Judge)},
  fonttitle=\bfseries,
  attach boxed title to top center={yshift=-2mm},
  boxed title style={sharp corners, size=small},
]
\small
\noindent \textbf{Task} \\
Decide if the final answer \textbf{matches} the ground truth.

\medskip
\noindent \textbf{Rules (both must hold)} \\
1) \textbf{Final-answer equivalence}: extract the last explicit final answer and compare it to the ground truth \emph{after trimming whitespace and ignoring case}. \\
2) \textbf{Reasoning validity }: Validity ($Z\rightarrow X\rightarrow Y$): Z must substantiate X, and Z,X must substantiate Y. Accept only reasoning that is logically sound, evidence-backed, and contradiction-free; apply near-adversarial scrutiny—any gaps, guesses, or inconsistencies $\rightarrow$ 0.

\bigskip
\noindent \textbf{Question Z} \\
\texttt{\{question\}}

\medskip
\noindent \textbf{Ground Truth Answer } \\
\texttt{\{correct\_answer\}}

\medskip
\noindent \textbf{Provided Reasoning X+Y} \\
\texttt{\{reasoning\_process\}}

\bigskip
\noindent \textbf{Output (plain-text JSON, no code fences)}
\begin{verbatim}
{
  "judge_explanation": "<brief explanation>",
  "result": "<Yes or No>"
}
\end{verbatim}

\end{tcolorbox}
\caption{Prompt template for strict binary judgment requiring answer equivalence and valid reasoning.}
\label{fig:binary_template}
\end{figure*}






\begin{table}[ht]
\centering
\caption{Welch’s $t$-tests comparing $S(A,B)$ distributions between Incorrect and Correct groups.}
\label{tab:sab_ttest}
\begin{tabular}{lcc}
\toprule
Metric & $t$ & $p$-value \\
\midrule
$S(Z,X)$ & $3.55$ & $0.0005$ \\
$S(X,Y)$ & $3.23$ & $0.0014$ \\
$S(Z,Y)$ & $3.15$ & $0.0018$ \\
\bottomrule
\end{tabular}
\vspace{0.5em}
\label{tab:sabtest}
\end{table}

\section{Dataset Features}


\noindent
\begin{table}[h]
\centering
\small
\caption{Training Dataset}
\resizebox{0.8\linewidth}{!}{
\begin{tabular}{p{2.5cm} p{6.8cm} p{3.7cm}}
\toprule[1pt]
\textbf{Dataset} & \textbf{Question} & \textbf{Answer} \\ 
\midrule
\textbf{BBEHCausal} & Question: Alice (AF) and Bob (BF) each fire a bullet at a window, simultaneously striking the window. The window only shatters (WS) if it is hit by two bullets. Is Alice firing the bullet a sufficient cause for the window shattering? Reply based on the answer a logician would give. & No \\ \hline
\textbf{CasHOLD} &Context: limitations was equitably tolled fails because he did not raise this argument before the district court. See Hinton v. Pac. Enters., 5 F.3d 391, 395 (9th Cir.1993) (HOLDING). Grace’s remaining contentions lack merit. Select the correct holding (1 to 4): "holding that the burden to allege facts sufficient to establish jurisdiction resides with plaintiff",
"holding that the plaintiff bears the burden when relying on the discovery rule",
"recognizing the validity of the doctrine but holding no equitable tolling on the facts presented",
"holding that plaintiff bears the burden to timely allege facts supporting equitable tolling",
"holding that the burden is on the plaintiff to allege facts sufficient to establish jurisdiction" & 3: "holding that plaintiff bears the burden to timely allege facts supporting equitable tolling" \\ \hline
\textbf{MATHHARD} & There are thirty-five red, yellow, orange, and white marbles in a bag. If half the number of red marbles equals two less than the number of yellow marbles, equals a third the number of orange marbles, and equals a third of three more than the number of white marbles, how many red marbles are there?& 8 \\ \hline
\textbf{IfQA} &If Caroline Flack's mother had felt the name suited her, what would have been Caroline's name? & Caroline\\
\bottomrule[1pt]
\end{tabular}} \vspace{-1mm}
\label{tab:dataset-examples1}
\end{table}

\vspace{-3mm}
\noindent
\begin{table}[h]
\centering
\small
\caption{Testing Dataset}
\resizebox{0.8\linewidth}{!}{
\begin{tabular}{p{2.5cm} p{6.8cm} p{3.7cm}}
\toprule[1pt]
\textbf{Dataset} & \textbf{Question} & \textbf{Answer} \\ 
\midrule
\textbf{BBEHMATH} & Consider the operations $a \mathbin{\text{[]}} b=\begin{cases}(a-b),&a+b>0\\ -a-{-}b,&\text{otherwise}\end{cases}$, $a \mathbin{;} b=\begin{cases}(b-a)\ast b,&\lvert a-b\rvert<2\\ (a-b)\ast a,&\text{otherwise}\end{cases}$, $a \mathbin{!} b=\begin{cases}(2 \mathbin{\text{[]}} b)\ast a,&a>b\\ b,&\text{otherwise}\end{cases}$, and $a \mathbin{@} b=\begin{cases}(a \mathbin{!} b),&a-b>0\\ a\ast b,&\text{otherwise}\end{cases}$; define $A=\big(((-5 \mathbin{;} -7)+(\text{2} \mathbin{@} -4)) \mathbin{\text{[]}} ((-9 \mathbin{\text{[]}} -2) \mathbin{!} (6\cdot \text{1}))\big) \mathbin{\text{[]}} \big(((\text{7}-\text{6}) \mathbin{\text{[]}} (6+{-}6)) \cdot ((-2 \mathbin{\text{[]}} 9) \mathbin{;} (-1-5))\big)$, $B=\big(((\text{9} \mathbin{;} \text{7}) \mathbin{\text{[]}} (6\cdot \text{7})) \mathbin{\text{[]}} ((\text{2} \mathbin{!} -2)+(\text{3} \mathbin{\text{[]}} -6))\big) \mathbin{;} \big(((\text{1}-8)+(-5 \mathbin{@} -5)) \mathbin{;} ((-4-10) \mathbin{;} (-3 \mathbin{@} 6))\big)$, and $C=\big(((-10 \mathbin{!} -1) \mathbin{!} (-5+{-}6)) \mathbin{@} ((-8+5)\cdot(-10+10))\big) \mathbin{;} \big(((-3-3) \mathbin{;} (2 \mathbin{\text{[]}} -5)) - ((2\cdot 9)-(-8 \mathbin{\text{[]}} 2))\big)$; compute $A+B-C$ (answer as a number).&-10038\\ \hline
\textbf{CLadder} & Context: For husbands that don't set the alarm, the probability of ringing alarm is 42\%. For husbands that set the alarm, the probability of ringing alarm is 51\%. Question: For husbands that set the alarm, would it be less likely to see ringing alarm if the husband had not set the alarm? & yes \\ \hline
\textbf{LegalBench} &Context: Overall, the percentages and correlation coefficients in nine of the ten endogenous are sufficient for purposes of the second Gingles precondition. See 478 U.S. at 56. Plaintiffs have shown minority political cohesiveness in these [**105] appellate-judgeship elections. Nonetheless, the stark statistic remains that, in Alabama's history, only two African Americans have been elected to statewide office for a total of [**186] three general elections and three primary elections. At the statewide level, this factor weighs in favor of Plaintiffs. Question: Based on the excerpt above, did the judge's legal reasoning rely on statistical evidence (e.g., regression analysis) when determining whether there was a causal link? Answer Yes or No & Yes \\ \hline
\textbf{LogiQA} & Some purple clay pots are alive. Therefore, some living things have good or bad quality.Question: Which of the following judgments, if true, would best strengthen the argument?Options: ['The quality of purple clay teapots is different from good to bad.', 'Some purple clay pots are inanimate.', 'There is no difference in the quality of purple clay teapots.', 'Some living things are not purple clay pots.'] & 0, The quality of purple clay teapots is different from good to bad. \\
\bottomrule[1pt]
\end{tabular} }\vspace{-2mm}
\label{tab:dataset-examples2}
\end{table}

\section{The Use of Large Language Models (LLMs)}
LLMs are the primary subject of our reinforcement learning (RL) experiments: we fine-tune pretrained models under CE-PO and baseline objectives to study stability, reward hacking, and causal alignment. Separately, we use LLMs as utilities for controlled text transformations—rephrasing, grammar correction, and during writing phase. 

\end{document}